\documentclass[journal,10pt]{IEEEtran}

\ifCLASSINFOpdf
\else
\fi
\usepackage{amsmath,amssymb,amsthm,comment,aliascnt}
\usepackage[utf8]{inputenc}
\usepackage{graphicx}
\usepackage{caption}
\usepackage{subcaption}
\DeclareGraphicsExtensions{.jpg,.mps,.pdf,.png,.eps} 
\usepackage{array}
\usepackage{nameref,hyperref}
\usepackage[numbers, comma, sort&compress]{natbib}
\usepackage{xcolor}
\usepackage{xspace}		
\usepackage{url}
\usepackage[normalem]{ulem}

\newcommand{\prox}{{\rm prox}}

\newcommand{\RR}{\mathbb{R}}

\newcommand{\PP}{\mathbb{P}}

\newtheorem{theorem}{Theorem}
\newaliascnt{proposition}{theorem}
\newtheorem{proposition}[proposition]{Proposition}
\aliascntresetthe{proposition}
\newaliascnt{lemma}{theorem}
\newtheorem{lemma}[lemma]{Lemma}
\aliascntresetthe{lemma}
\newaliascnt{corollary}{theorem}

\aliascntresetthe{corollary}
\newaliascnt{example}{theorem}
\newtheorem{example}[example]{Example}
\aliascntresetthe{example}
\newaliascnt{remark}{theorem}
\newtheorem{remark}[remark]{Remark}
\aliascntresetthe{remark}

\newcommand{\modif}[1]{\textcolor{black}{#1}}

\def\rset{\mathbb{R}}
\def\nset{\mathbb{N}}

\def\eqdef{:=}
\def\eqsp{\;}
\def\rmd{\mathrm{d}}
\def\1{\mathbbm{1}}

\newcommand{\ooint}[1]{\left(#1\right)}
\newcommand{\ccint}[1]{\left[#1\right]}

\def\Z{\mathsf{Z}}
\def\Zvect{{\boldsymbol{\mathsf{Z}}}}
\def\Zphi{\Phi^{\Zvect}}
\def\R{\mathsf{R}}
\def\Rvect{\boldsymbol{\mathsf{R}}}

\def\O{\mathsf{O}}
\def\Ovect{\boldsymbol{\mathsf{O}}}
\def\param{\theta}
\def\paramvect{{\boldsymbol{\theta}}}
\def\varthetavect{{\boldsymbol{\vartheta}}}

\def\vect0{\modif{\boldsymbol{\mathsf{0}}}}
\def\Dset{\mathcal{D}}
\def\Sset{\mathcal{S}}
\def\rset{\mathbb{R}}
\def\nset{\mathbb{N}}
\def\lambdatime{\lambda_{\mathsf{R}}}
\def\lambdaO{\lambda_{\mathsf{O}}}
\def\pas{\gamma}
\def\C{\modif{\boldsymbol{\mathsf{C}}}}
\def\D2{\modif{\boldsymbol{\mathsf{D}}}}

\def\barD2{\modif{\overline{\boldsymbol{\mathsf{D}}}}}
\def\barDinv{\modif{\overline{\boldsymbol{\mathsf{D}}}_{\tt I}}}
\def\barDo{\modif{\overline{\boldsymbol{\mathsf{D}}}_{\tt O}}}
\def\U{\modif{\boldsymbol{\mathsf{U}}}}
\def\A{\modif{\boldsymbol{\mathsf{A}}}}
\def\barA{\modif{\bar{\boldsymbol{\mathsf{A}}}}}

\def\I{\modif{\boldsymbol{\mathsf{I}}}}
\def\Id{\modif{\boldsymbol{\mathsf{I}}}}

\def\0mat{\modif{\boldsymbol{\mathsf{0}}}}

\def\PE{\mathbb{E}}
\def\PP{\mathbb{P}}
\def\stepalgo{\boldsymbol{\mu}}
\usepackage{hyperref,cleveref,bbm}
\usepackage[linesnumbered,vlined,ruled,commentsnumbered]{algorithm2e}
\newtheorem{assumption}{A\hspace{-3.1pt}}
\Crefname{assumption}{A\hspace{-3pt}}{A\hspace{-3pt}}
\crefname{assumption}{A}{A}

\usepackage{xr,newfile}

\begin{document}
%
\title{Covid19 Reproduction Number: Credibility Intervals by Blockwise Proximal Monte Carlo Samplers}
%
%
%

\author{G. Fort,  B. Pascal~\IEEEmembership{Member,~IEEE}, P. Abry~\IEEEmembership{Fellow,~IEEE}, N. Pustelnik~\IEEEmembership{Member,~IEEE}.
\thanks{Manuscript received February 3, 2022; first revision December 22, 2022,  second revision February 14, 2023. \textit{Corresponding author: Gersende Fort.}}
\thanks{Gersende Fort is with CNRS, Institut de Math\'ematiques de Toulouse, Toulouse, France (e-mail: gersende.fort@math.univ-toulouse.fr). Work partly supported by the {\it Fondation Simone et Cino Del Duca, Institut de France}.}
\thanks{This paper has supplementary downloadable material available at \url{http://ieeexplore.ieee.org.} provided by the authors.
The material includes detailed proofs of theoretical results. This material is 300~ko in size.}
\thanks{Barbara Pascal is with Nantes Université, École Centrale Nantes, CNRS, LS2N, UMR 6004, F-44000 Nantes, France (e-mail: barbara.pascal@cnrs.fr).}
\thanks{Patrice Abry and Nelly Pustelnik are with CNRS, ENS de Lyon, Laboratoire de Physique, Lyon, France  (e-mail: firstname.lastname@ens-lyon.fr).}
}

%
%

\markboth{Journal of \LaTeX\ Class Files,~Vol.~XX, No.~X, Month~2021}%
{Shell \MakeLowercase{\textit{et al.}}: Bare Demo of IEEEtran.cls for IEEE Journals}
%



\maketitle

\begin{abstract}
Monitoring the Covid19 pandemic constitutes a critical societal stake that received considerable research efforts.  The intensity of the pandemic on a given territory is efficiently measured by the reproduction number, quantifying the rate of growth of daily new infections.  Recently, estimates for the time evolution of the reproduction number were produced using an inverse problem formulation with a nonsmooth functional minimization.  While it was designed to be robust to the limited quality of the Covid19 data (outliers, missing counts), the procedure lacks the ability to output credibility interval based estimates. This remains a severe limitation for practical use in actual pandemic monitoring by epidemiologists that the present work aims to overcome by use of Monte Carlo sampling. After interpretation of the nonsmooth functional into a Bayesian framework, several sampling schemes are tailored to adjust the nonsmooth nature of the resulting posterior distribution.  The originality of the devised algorithms stems from combining a Langevin Monte Carlo sampling scheme with Proximal operators.  Performance of the new algorithms in producing relevant credibility intervals for the reproduction number estimates and denoised counts are compared.  Assessment is conducted on real daily new infection counts made available by the Johns Hopkins University.  The interest of the devised monitoring tools are illustrated on Covid19 data from several different countries.

\end{abstract}


%
\IEEEpeerreviewmaketitle

\section{Introduction}
\label{sec:introduction}
\noindent {\bf Context.}  The Covid19 pandemic is causing unprecedented health, social, and economic crises.
This triggered massive research efforts to design efficient
procedures aiming to assess the intensity of the pandemic, a
prerequisite to develop efficient sanitary
policies~\cite{flahault2020covid}.  Several indices are commonly used
to measure the strength of a pandemic, such as, e.g., the reproduction
number of interest here.  However, often, the value of the index alone
is not sufficient and credibility intervals of these indices
constitute valuable information for the decision makers, notably in
periods of rapid pandemic evolution or of changes in trends, an issue
not always addressed in pandemic monitoring and at the heart of the
present work. \\
\noindent {\bf Related works.} Pandemic monitoring can be conducted with numerous tools from different scientific fields, (cf. \cite{arino2021describing} for a review), amongst which {\it compartmental models}, such as the founding {\it Susceptible-Infectious-Recovered} scheme.
Within pandemic period, when data are scare and of limited quality,
the reproduction number, $\R_t$, is often used by epidemiologists as
an efficient practical proxy for the pandemic intensity: it measures
the number of second infections caused by one primary infection
(cf. e.g., \cite{Diekmann1990,wallinga2004,vandenDriessche2002,obadia2012,cori2013new}).
It thus plays a key role in the pandemic evolution assessment: \modif{the}
number of new infections today, $\Z_t$, depends on $\R_t$ and on \modif{a weighted}
average of the new infection counts on previous days
$\{\ldots, \Z_{t-3}, \Z_{t-2}, \Z_{t-1} \}$; \modif{the weight function being}  the
so-called {\it serial interval function} $\Phi$, that
quantifies the distribution of the random delays between the onsets of symptoms in a primary and secondary cases
\cite{cori2013new,obadia2012,thompson2019,Liu2018}.
It has recently been shown that, within pandemic, reliable estimates
for the temporal estimation of $\R_t$ can be obtained from an inverse
problem formulation resulting in a nonsmooth convex optimization
problem \cite{abry2020spatial,Pascal2021}.  The functional to minimize
is built from combining the pandemic model
in \cite{cori2013new}, with time regularity constraints.  While
the procedure was engineered to produce realistic estimations of the
temporal evolution of the reproduction number that are robust to the
limited quality of the Covid19 pandemic data (severely corrupted with
outliers, missing or negative counts and pseudo-seasonalities), it
does not however provide credibility intervals, a critical issue
towards its practical and actual use by epidemiologists, \modif{the main issue} that we aim to
address in the present work. \\
\noindent {\bf Goals, contributions and outline.} The overall goal of the present work is to devise Monte Carlo sampling strategies to perform the estimation by means of credibility intervals of the pandemic reproduction number and of denoised infection counts. 
To that end, Section~\ref{sec:model} details the proposed  \modif{statistical}
model used to embed into a stochastic framework the epidemiological
model in \cite{cori2013new} and its robust extension to data
corruption \cite{Pascal2021}.  Its originality stems from using
non-differentiable priors to ensure robustness to data
corruption. The uniqueness of the maximum a posteriori is thoroughly studied.
Further, Section~\ref{sec:sampler} devises original
sampling schemes tailored to handle the non-differentiability of the
target distribution. We propose two blockwise Proximal-Gradient based
extensions of the Langevin Metropolis algorithms: {\tt PGdec} and {\tt
PGdual}.  We establish their ergodicity, and carry out a comparative
study. \modif{  Using real Covid19 data, made available at the}
Johns Hopkins University repository and described in
Section~\ref{sec:data}, the performance of up to twelve variations of
the sampling strategies are assessed and compared, using well-thought
indices quantifying their efficiency
(cf. Section~\ref{sec:assessment}).  Finally, in
Section~\ref{sec:results}, the relevance of the proposed blockwise
Proximal-Gradient samplers is illustrated for several different
countries representative of the evolution of the pandemic across the
world, for a 5-week recent period. 
Daily updates of these credibility interval estimates as well as
MATLAB routines for their calculations are available at \url{https://github.com/gfort-lab/OpSiMorE}.

\noindent {\bf Notations.} Vectors are column-vectors, \modif{and are denoted in  bold font, e.g., $\Rvect \in \mathbb{R}^T$}. For $p \leq q$, the
vector $\modif{\boldsymbol{\mathsf{x}}}_{p:q}$ concatenates the scalars $\modif{\mathsf{x}}_i$ 
or vectors $\modif{\boldsymbol{\mathsf{x}}}_i$ for $i
=p, \ldots, q$.  \modif{The matrices are denoted in bold font}. For a matrix $\A$,  $\A^\top$
(resp. $\mathrm{det}(\A)$ and $\A^{-1}$) denotes the transpose of $\A$
(resp. the determinant and the inverse of $\A$). We set
$\A^{-\top} \eqdef (\A^\top)^{-1} = (\A^{-1})^\top$. $\Id_{p}$ is the
$p \times p$ identity matrix, and $\0mat_{p \times q}$ is the $p \times q$
null matrix \modif{($\0mat_{p \times 1}$ will be denoted by $\0mat_p$)}.  For a vector $\modif{\boldsymbol{\mathsf{x}}} \in \rset^p$, $\|\modif{\boldsymbol{\mathsf{x}}}\|_1$ is the
$L^1$-norm and $\|\modif{\boldsymbol{\mathsf{x}}}\|$ is the $L^2$-norm.  Finally,
$\mathcal{N}_r(\modif{\boldsymbol{\mu}}, \C)$ denotes the $\rset^r$-valued Gaussian
distribution with expectation $\modif{\boldsymbol{\mu}}$ and covariance matrix $\C$.  For
some $\gamma>0$, the proximity operator of a proper, convex,
lower semi-continuous function $f$ from $\mathbb{R}^d$ to
$]-\infty,+\infty]$  is defined
as
$$
(\forall \modif{\boldsymbol{\mathsf{x}}} \in \mathbb{R}^d)\, \quad  \prox_{\gamma f}(\modif{\boldsymbol{\mathsf{x}}}) \eqdef \arg\min_{\modif{\boldsymbol{\mathsf{y}}} \in \rset^d} \gamma f(\modif{\boldsymbol{\mathsf{y}}}) + \frac{1}{2} \Vert \modif{\boldsymbol{\mathsf{y}}}-\modif{\boldsymbol{\mathsf{x}}}\Vert^2. 
$$

\section{Covid19 pandemic Bayesian Model}
\label{sec:model}
\subsection{Pandemic model}
The present work makes use of a pandemic model devised by
epidemiologists in \cite{cori2013new} that focuses on a main pandemic
index: the reproduction numbers, to be estimated from daily
new infection counts.  Elaborating on \cite{cori2013new}, it
was further proposed in \cite{Pascal2021} to account for the limited
quality of the intra-pandemic Covid19 data - highly corrupted by
irrelevant, missing and mis-reported counts or by pseudo-seasonal
effects - by means of additional {\it outliers} $\Ovect$, also unknown
and to be estimated. The goal of the present work is thus to
estimate, from a vector of $T$ observed  daily new infection counts
$\Zvect \eqdef (\Z_1, \ldots, \Z_T)^\top \in \nset^T$, the vector of
unknowns \[ (\Rvect, \Ovect)\eqdef \left((\R_1, \ldots, \R_T)^\top,(\O_1, \ldots, \O_T)^\top\right) \in { (\rset_+)^T \times \rset^T} \]
gathering the reproduction numbers $\R_t$ and the outliers $\O_t$ of $T$ consecutive days.

\subsection{A statistical model}
\label{sec:statistic:model}
Estimation entails the recourse to a formulation of the pandemic
model, where the unknown parameters
$$\paramvect\eqdef(\Rvect,\Ovect),$$ are a realization of a random
vector; the statistical model defines its probability distribution.  We
consider a density with respect to (w.r.t.) the Lebesgue measure on
$(\rset_+)^T \times \rset^T$ of the form
\begin{equation}
  \label{eq:target:covid}
\paramvect \mapsto \pi_\Zvect(\paramvect) \propto \exp\left( -f_\Zvect(\paramvect) - g(\paramvect) \right)
\1_{\Dset_{\Zvect}}(\paramvect),
\end{equation}
where $\1_{A}$ denotes the $\{0,1\}$-valued indicator function of the
set $A$. We choose $g, f_\Zvect, \Dset_\Zvect$ such that $\paramvect
\mapsto -\ln \pi_\Zvect(\paramvect)$ on $\Dset_\Zvect$ is the
criterion proposed in \cite{Pascal2021} for the estimation of
$\paramvect$ by minimization of a contrast.  The function $g$ is given by
\begin{equation}\label{def:fandg:g} 
  g(\paramvect) \eqdef \lambdatime \| \D2 \Rvect \|_1 +
  \lambdaO \| \Ovect \|_1 \eqsp;
 \end{equation}
 $\D2$ is proportional to the $(T-2) \times T$ discrete-time second
order derivative matrix:
\begin{equation} \label{eq:def:D2}
\D2 \eqdef  \frac{1}{\sqrt{6}}\left[ \begin{matrix}1 & -2 & 1 & 0 & 0 & \ldots & 0  \\
0 & 1 & -2 & 1 & 0 & \ldots & 0 \\
\ldots & & & & & & \ldots \\  
0 &  \ldots &  & &  1 & -2 & 1   \end{matrix} \right].
\end{equation}
 The function $f_\Zvect$ is given by 
\begin{equation}\label{def:fandg:f} 
  f_\Zvect(\paramvect) \eqdef \sum_{t=1}^T  \left\{-  \Z_t \, \ln \mathcal{I}_t(\paramvect, \Zvect)  + \mathcal{I}_t(\paramvect, \Zvect) \right\}, 
\end{equation}
where $\mathcal{I}_t(\paramvect, \Zvect) $ is a positive {\it
  intensity}.  Up to an additive constant, $f_\Zvect(\paramvect) =
\sum_t d_{\mathtt{KL}}(\Z_t \lvert \mathcal{I}_t(\paramvect, \Zvect)
)$, where $d_{\mathtt{KL}}$ denotes the Kullback-Leibler divergence
related to the log-likelihood of a Poisson process\footnote{Its
  definition is, for some $z\in \nset$,
\begin{align*}
(\forall \mathcal{I} \in \RR)\;\; d_{\mathtt{KL}}( z\lvert \mathcal{I}) \eqdef & \begin{cases}
z \ln \frac{z}{\mathcal{I}} + \mathcal{I}- z \quad &\text{if} \, \,  z > 0,  \,\mathcal{I} > 0,\\
\mathcal{I} \quad &\text{if} \, \, z = 0, \,  \mathcal{I} \geq 0,\\
+ \infty \quad &\text{otherwise.} 
\end{cases}
\end{align*}
}. We choose
\begin{equation}\label{eq:intensity}
\quad \mathcal{I}_t(\paramvect, \Zvect)  \eqdef \R_t \sum_{u = 1}^{\tau_\phi} \Phi_u \Z_{t-u} +
 \O_t.
\end{equation}
$\Phi \eqdef (\Phi_u)_{1 \leq u \leq \tau_\phi}$ is the {\it serial
interval function}, describing the average infectiousness profile after
infection \cite{cori2013new,thompson2019,Liu2018}.  $\Phi$ is assumed
known and, following \cite{Riccardo2020,Guzzetta}, is classically
modeled as a Gamma distribution truncated over $\tau_\phi = 26$ days
with mean and standard deviation of 6.6 and 3.5 days.  The
reproduction number $\R_t$ at time $\# t$ acts on the rate at which a
person infected at time $t-u$, generates new infections at time $t$:
this rate is equal to $\R_t\Phi_u$.

The support of the density $\pi_\Zvect(\cdot)$ is the measurable
subset $\Dset_{\Zvect}$ of $(\rset_+)^T \times \rset^T$ ensuring that
the intensities $\mathcal{I}_t(\paramvect, \Zvect) $ are positive, or non-negative when $\Z_t =
0$ (with the convention that $0 \ln 0 = 0$ in the expression of $f_\Zvect$). This yields
\begin{multline}  \label{eq:const}
\Dset_{\Zvect} \eqdef \{\paramvect: \mathcal{I}_t(\paramvect, \Zvect)
>0 \ \text{for $t$ s.t. $\Z_t >0$} \} \\ \cup \{\paramvect:
\mathcal{I}_t(\paramvect, \Zvect) \geq 0 \ \text{for $t$ s.t. $\Z_t =
  0$}\}.  
\end{multline}

\subsection{A Bayesian model}
\label{sec:bayesianmodel}
Let us describe a Bayesian framework such that
$\pi_\Zvect$ is a {\it  posterior} distribution. The model depends
on initial values $\Zvect_{1-\tau_\phi:0}, \R_1, \R_2$, which are omitted in
the notations.

\noindent \textbf{Prior distribution.}  Conditionally to $(\R_1,
\R_2)$, the prior distribution of $(\Rvect_{3:T}, \Ovect)$ has the
density
\[
 \left( \prod_{t=3}^T p_1(\R_t \vert \R_{t-1}, \R_{t-2})\right) \, \left( \prod_{t=1}^T p_2(\O_t)\right) 
 \]
w.r.t. the Lebesgue measure on $\rset^{T-2} \times \rset^T$; we set
\begin{align*}
  p_1(r_t \vert r_{t-1}, r_{t-2}) &\! \eqdef \! \frac{\lambdatime}{2 \sqrt{6}}\exp\left(- \frac{\lambdatime}{\sqrt{6}} \vert
r_t- 2 r_{t-1} + r_{t-2}\vert  \right) \\
  p_2(o_t) & \eqdef \frac{\lambdaO}{2} \exp(-\lambdaO  \, \vert o_t \vert). \label{eq:priorO}
\end{align*}
It is easily seen that this  prior density is equal to
$\exp(-g(\paramvect))$ up to a multiplicative constant (depending on
$\lambdatime$ and $\lambdaO$).  Such a prior states that
$\Rvect_{3:T}$ and $\Ovect$ are mutually independent conditionally to
$(\R_1, \R_2)$. Further, the outliers $\Ovect$ are independent and
distributed as a Laplace distribution (with parameter $\lambdaO>0$) as
commonly encountered in the literature (see, e.g.,
~\cite{moulin:liu:1999,figueiredo:2003,park:casella:2008}). Finally
$\Rvect_{3:T}$ is distributed as an AR(2) process with Laplace
distribution (with parameter $\lambdatime>0$).  $\lambdatime$ and $\lambdaO$ are (fixed) positive \textit{regularization
  hyperparameters}. $g$ favors some large values among many small ones
of the components of both vectors $\D2 \Rvect$ and $\Ovect$; it models
smooth piecewise linear time evolutions for $\Rvect$, or equivalently
a sparse set of components where the discrete second order derivative
in time of $\Rvect$ is non zero.

\noindent \textbf{Likelihood.}  Conditionally to $(\Rvect, \Ovect)$,
the observations are not independent and the joint distribution of
$\Zvect$ is
\[
 \prod_{t=1}^T p_3(\Z_t \vert \Zvect_{t-\tau_\phi:t-1},  \R_t, \O_t).
\]
When $\R_t \geq 0$ and $\mathcal{I}_t(\paramvect, \Zvect) >0$ (see \eqref{eq:intensity} for the definition of $\mathcal{I}_t$), following \cite{cori2013new,Pascal2021},
the pandemic diffusion is modeled as a Poisson distribution with
time-varying intensity $\mathcal{I}_t(\paramvect, \Zvect) $:
\begin{multline}
\label{eq:likelihood}
  p_3(\Z_t\vert {\Zvect_{1:t-1}}, \R_t, \O_t)  
  \eqdef \frac{\big(\mathcal{I}_t(\paramvect, \Zvect) \big)^{\Z_t}}{\Z_t!}  \exp(-
 \mathcal{I}_t(\paramvect, \Zvect) ).
\end{multline}
Such a model claims that the mean value of the counts $\Z_t$ at time
$t$ is $\sum_{u=1}^{\tau_\phi} \R_t \Phi_u \Z_{t-u} + \O_t$.  When
$\R_t \geq 0$ and $\mathcal{I}_t(\paramvect, \Zvect) =0$, $p_3$ is the
Dirac mass at zero: it is equal to \eqref{eq:likelihood} when $\Z_t=0$
and equal to $0$ otherwise.  Finally, when $\R_t < 0$ or
  $\mathcal{I}_t(\paramvect, \Zvect) < 0$, $p_3$ is a distribution on
  the negative integers; this implies that when $\Zvect \in \nset^T$,
  $p_3(\Z_t \vert {\Zvect_{t-\tau_\phi:t-1}}, \R_t, \O_t) = 0$.

This description implies that the likelihood of a vector of
observations $\Zvect$ taking values in $\nset^T$ is $\exp(-
f_\Zvect(\paramvect)) \1_{\Dset_\Zvect}(\paramvect)$.

\noindent \textbf{Posterior distribution.} This model implies that
$\pi_\Zvect(\paramvect)$ is the a posteriori distribution of
$(\Rvect_{3:T}, \Ovect)$ given $\Zvect$ when $\Zvect \in \nset^T$.

\subsection{Bayesian estimators} In the Bayesian approach to
Decision Theory, the maximum, the median, and the expectation of the a
posteriori distribution, are Bayes estimators $\widehat{\paramvect}$
associated to a loss function $\ell$
\[
 \widehat{\paramvect}(\Zvect) \eqdef \mathrm{Argmin}_{{\boldsymbol \tau} \in \Dset_\Zvect} \int_{\Dset_\Zvect} \ell({\boldsymbol \tau},\paramvect) \pi_\Zvect(\paramvect) \rmd \paramvect;
\]
$\ell$ is, respectively, the $0-1$ loss, the $L^1$-norm and the
squared $L^2$-norm (see e.g. \cite[Sections 2.3. and
  2.5.]{robert:bayesianchoice}).   Computing the Maximum a Posteriori
(MAP) fits the minimization problem proposed in \cite{Pascal2021} for
the reconstruction of $\paramvect$: \begin{align}
\label{eq:minPascal}
\underset{\Rvect, \Ovect}{\mathrm{Argmin}} & \sum_{t=1}^T d_{\mathtt{KL}}(\Z_t \lvert \mathcal{I}_t(\paramvect, \Zvect) ) + \lambdatime \| \D2 \Rvect \|_1 +
\lambdaO \| \Ovect \|_1 \eqsp.
\end{align}
The optimization problem is a nonsmooth convex minimization problem both
encapsulating  the transmission process, favoring piecewise linear
behavior of the reproduction number along time and sparsity of the
outliers. The minimization is performed with the Chambolle-Pock
primal-dual algorithm allowing to handle both the
non-differentiability and linear operators
\cite{bauschke2011convex,Chambolle2011first}.  {Properties of the
  MAP are established in
  \Cref{prop:MAPexists}.}
\begin{proposition} \label{prop:MAPexists}
 If there are at least two positive averaged counts
 $\sum_{u=1}^{\tau_\phi} \Phi_u \Z_{t_\star-u}$,
 $\sum_{u=1}^{\tau_\phi} \Phi_u \Z_{t_{\star \star}-u}$, and one
 positive count $\Z_\tau$, a MAP exists.
   If $\paramvect^\star=(\Rvect^\star,
   \Ovect^\star)$ and $\paramvect^{\star \star} = (\Rvect^{\star
     \star}, \Ovect^{\star \star})$ are two MAP, then  $\mathcal{I}_t(\paramvect^{\star}, \Zvect) = \mathcal{I}_t(\paramvect^{\star \star}, \Zvect)$, $\O_t^\star \,
   \O_t^{\star \star} \geq 0$ and $(\D2 \Rvect^\star)_t \, (\D2
   \Rvect^{\star \star})_t \geq 0$ for any $t \in
   \{1, \ldots, T \}$.
\end{proposition}
\begin{proof}
The first statement is adapted from \cite{Pascal2021}; 
  The second statement is established in
  \cite{Pascal2021}. The sign conditions result from a first order
  expansion of the $L^1$-norm.  For a detailed proof, see
  \autoref{sec:SM:proofMAP} in the Supplementary material.
  \end{proof}
  
{\Cref{prop:MAPexists} implies that the MAP is either unique, or that  there are uncountably many MAP.
In addition, it shows
that $f_{\Zvect}$ and $g$ are constant over the set of the minimizers.}
{Thus, following the same lines as in \cite{ali:tibshirani:2019}, a
  sufficient condition for the uniqueness of the MAP is derived
  (see \autoref{sec:SM:proofMAP} in the Supplementary material).}

The expression of the  distribution $\pi_{\Zvect}$ in 
\eqref{eq:target:covid} is so complex that  it is known
only up to a normalizing constant. Consequently, the computation of
most statistics of $\pi_\Zvect(\cdot)$ relies on Monte Carlo
samplers, in order to produce  samples $\{\paramvect^n, n \geq 0 \}$ in
$\Dset_{\Zvect}$ {\it approximating} $\pi_\Zvect$ (see e.g. \cite[section
  2.3]{douc:etal:2018}): for example, the estimation of the median and
more generally the quantiles of $\pi_\Zvect$ can rely on
the order statistics of the samples, and the mean a posteriori can be
approximated by the Monte Carlo sum $N^{-1} \sum_{n=1}^N
\paramvect^n$.

\section{Blockwise Proximal-Gradient Monte Carlo samplers}
\label{sec:sampler}

The aim is now to devise Monte Carlo sampling strategies for the distribution defined in~\eqref{eq:target:covid}.
However, this section will address a broader class of densities, defined on $\RR^d$ with respect to the Lebesgue measure,
and expressed as $\pi(\paramvect)\propto \exp(-F(\paramvect)  )\1_{\Dset}(\paramvect)$ where $F\eqdef f + g$  and $f$, $g$, and $\Dset$ satisfy the
smoothness and blockwise structure assumptions \Cref{ass:smooth} and
\Cref{ass:nonsmooth} defined below. 

\subsection{Smoothness and Blockwise structure}
\begin{assumption}
  \label{ass:smooth} $f$ and $g$ are finite on $\Dset \subseteq \rset^d$ and  $f$ is continuously differentiable on the interior of $\Dset$.
\end{assumption}
\noindent Additionally, $g$ has a blockwise structure that we aim to 
use
in the design of the proposed samplers.  This blockwise structure
stems both from the decomposition of $\paramvect$ into $J$ blocks
$(\paramvect_1,\ldots, \paramvect_J)\in \rset^{d_1}\times \ldots
\times \rset^{d_J}$ and from the sum of several functions of $\paramvect_j$
  possibly combined with a linear operator.

\begin{assumption} \label{ass:nonsmooth} For $j \in \{1,\ldots, J\}$, $i \in \{1,\ldots, I_j\}$, there exist matrices $\A_{i,j} \in \mathbb{R}^{c_{i,j} \times d_j}$,   and proper, convex,
lower semi-continuous functions $g_{i,j} \colon\mathbb{R}^{c_{i,j}}
\to$ $ ]-\infty, +\infty]$ such that $\sum_{j=1}^J d_j = d$ and
\[
\forall \paramvect \eqdef (\paramvect_1^\top, \ldots, \paramvect_J^\top)^\top,
\qquad g(\paramvect) \eqdef \sum_{j = 1}^{J} \sum_{i=1}^{I_j}
g_{i,j}(\A_{i,j} \paramvect_j)\eqsp.
\]
In addition, the proximity operator of $g_{i,j}$ has a closed form expression.
\end{assumption}

In Bayesian inverse problems,
$f$ may stand for the data fidelity term and the nonsmooth part $g$
stands for many penalty terms acting on blocks of the parameter
$\paramvect$.  Different splittings of the prior defined
by \eqref{def:fandg:g} fits Assumption~A\ref{ass:nonsmooth}.

\begin{example} \label{ex:PGDcase1} The prior $g$ given by \eqref{def:fandg:g} satisfies \Cref{ass:nonsmooth}: \[
\lambdatime \|\D2 \Rvect\|_1 + \lambdaO \|\Ovect\|_1 =
g_{1,1}(\A_{1,1} \Rvect) + g_{1,2}(\A_{1,2}\Ovect),
\]
where $\paramvect_1 \eqdef \Rvect$, $\paramvect_2 \eqdef \Ovect$,
$\A_{1,1} \eqdef \D2$, $\A_{1,2} = \Id_T$, $g_{1,1} \eqdef \lambdatime\Vert
\cdot \Vert_1$, and $g_{1,2} \eqdef \lambdaO\Vert \cdot \Vert_1$.
    \end{example}

\begin{example} \label{ex:PGDcase2} The prior $g$ given by \eqref{def:fandg:g} satisfies \Cref{ass:nonsmooth}: \[
\lambdatime \|\D2 \Rvect\|_1 + \lambdaO \|\Ovect\|_1 = \sum_{i=1}^3 g_{i,1}(\A_{i,1} \Rvect) + g_{1,2}(\A_{1,2}\Ovect), 
\]
where $\paramvect_1 \eqdef \Rvect$ and $\paramvect_2 \eqdef
\Ovect$. $\A_{i,1}$ collects the rows $i, i+3, i+6, \ldots,$ of the
matrix $\D2$, $\A_{1,2} \eqdef \Id_T$, $g_{i,1} \eqdef \lambdatime
\|\cdot\|_1$ and $g_{1,2} \eqdef \lambdaO \| \cdot \|_1$.
    \end{example} 
\noindent The second example follows block splitting strategies described in \cite{Pustelnik_N_2011_ieee-tip_par_pai, pascal_block_2018}.

\subsection{Proximal algorithms and Metropolis-Hastings algorithms}
\label{sec:proxMCMC}
The design of an optimization strategy to minimize $F$ on $\Dset$ and
the design of a sampler  to  approximate the target distribution
$\pi$ both rely on the activation of an operator
$\stepalgo\colon \rset^{d} \to \rset^d$. To be more specific, when
minimizing $F$, we aim to design a sequence of the form:
\begin{equation}
\label{eq:algoMPA}
\paramvect^{n+1} = \stepalgo( \paramvect^{n})
\end{equation}
where $\stepalgo$ is an operator built from $F$ and $\Dset$ in such a
way that the sequence $(\paramvect^{n})_{n\in \mathbb{N}}$ converges
to a minimizer of $F$ (cf. \cite{Bauschke:2011ta,Combettes2011,Chambolle_A_2016_an} for an exhaustive list of
algorithmic schemes). When building a Metropolis-Hastings algorithm
(say with Gaussian proposal), a new point is proposed as:
\begin{equation} \label{eq:Langevin:dynamic}
\paramvect^{n+1/2} = \stepalgo( \paramvect^{n}) + \, \boldsymbol{\xi}^{n+1} \eqsp
\quad \mbox{where} \quad \boldsymbol{\xi}^{n+1} \sim {\mathcal{N}_d(\vect0_{d}, \C)};
\end{equation}
 $\C\in \rset^{d \times d}$ is a positive definite matrix. The
Langevin dynamics is recovered in the specific case where $F$ is
smooth with $\stepalgo(\paramvect) \eqdef \paramvect - \pas \nabla
F(\paramvect)$ being a gradient ascent over $\ln \pi$ with step size
$\pas>0$ and $\C \eqdef {2 \pas}\Id_d$.  Scaled Langevin samplers are
also popular: given a $d \times d$ matrix $\boldsymbol{\Gamma}$, set
\begin{equation} \label{eq:LangevinT:dynamic}
\stepalgo(\paramvect) \eqdef \paramvect - \pas \boldsymbol{\Gamma} \boldsymbol{\Gamma}^\top
\nabla F(\paramvect), \qquad \C \eqdef {2 \pas} \, \boldsymbol{\Gamma} \boldsymbol{\Gamma}^\top;
\end{equation}
they are
inherited from the so-called {\it tempered Langevin
  diffusions}~\cite{kent:1978} (see also \cite{roberts:stramer:2002}
for a pioneering work on its use in the Markov Chain Monte Carlo
literature).  Either this proposed point is the new point
$\paramvect^{n+1} = \paramvect^{n+1/2}$ (thus yielding the {\it
  Langevin Monte Carlo} algorithm~\cite{parisi:1981}; see also
\cite{durmus:moulines:2016,dwivedi:2019}) or there is an
acceptance-rejection Metropolis mechanism (thus yielding the {\it
  Metropolis Adjusted Langevin Algorithm} (MALA)
\cite{roberts:tweedie:1996}).
The general Metropolis-Hastings procedure with Gaussian proposal, is
summarized in Algorithm~\ref{algo:MALA} where we denote by
$q(\paramvect, \boldsymbol{\tau} )$ the density of the distribution
$\mathcal{N}_d(\stepalgo(\paramvect), \C)$ evaluated at
$\boldsymbol{\tau} \in
\rset^{d}$:
\[
q(\paramvect,\boldsymbol{\tau}) \eqdef \frac{\exp\left(-0.5 \, (\boldsymbol{\tau}
  - \stepalgo(\paramvect))^\top \,
  \C^{-1} (\boldsymbol{\tau} -
  \stepalgo(\paramvect))\right)}{\sqrt{2\pi}^{d}
  \sqrt{\mathrm{det}(\C)}}.
\]
The constraint $\paramvect \in \Dset$ is managed by the
  acceptance-rejection step (since $\pi(\paramvect^{n+1/2}) =0$ when
  $\paramvect^{n+1/2} \notin \Dset$) but not necessarily in the
  proposal mechanism.
\begin{algorithm}[htbp]
  \KwData{a positive definite matrix $\C$; $\pas
    >0$; a positive integer $N_{\mathrm{max}}$; $\paramvect^{0} \in
    \Dset$} \KwResult{A $\Dset$-valued sequence $\{\paramvect^n, n \in
    \{0, \ldots, N_{\mathrm{max}} \} \}$} \For{$n=0, \ldots,
    N_{\mathrm{max}}-1$}{
    Sample
    $\boldsymbol{\xi}^{n+1} \sim \mathcal{N}_d(\modif{\0mat_d}, \C)$;\\ 
    Set $\paramvect^{n+\frac{1}{2}} = \stepalgo(\paramvect^n) + \boldsymbol{\xi}^{n+1}$\;
    Set $\paramvect^{n+1} = \paramvect^{n+\frac{1}{2}}$ with
    probability \begin{equation} \label{eq:AR:PGD}
    1 \wedge \frac{\pi(\paramvect^{n+\frac{1}{2}})}{\pi(\paramvect^n)}
\frac{q(\paramvect^{n+\frac{1}{2}}, \paramvect^n)}{q(\paramvect^n, \paramvect^{n+\frac{1}{2}})}
\end{equation}   and $\paramvect^{n+1} = \paramvect^n$ otherwise. 
  } \caption{Metropolis-Hastings with Gaussian
    proposal\label{algo:MALA}}
\end{algorithm}
The MALA algorithms drift the proposed moves towards areas of high
probability for the distribution $\pi$, using first order information
on $\pi$. Building on this idea, many strategies were proposed in the
literature in the setting defined by \Cref{ass:smooth}: $\stepalgo$
can either be a gradient step when $F$ is smooth, or a proximal step
(i.e., $\stepalgo = \prox_{\gamma F}$ also referred as an implicit
subgradient descent step), or a Moreau-Yosida envelope gradient step
(i.e. $\stepalgo = \modif{\Id_d} - \gamma \nabla (^\gamma F)$ where the Moreau envelope of a function $F$ with parameter $\gamma>0$ is defined as 
$
^{\gamma}F \eqdef \inf_y  \gamma F(y) + \frac{1}{2} \Vert \cdot-y \Vert^2
$).
In \cite{durmus:etal:2019}, explicit subgradient steps possibly
combined with a proximal step are used. 
In \cite{chatterji:etal:2020}, $\stepalgo$ relies on a Gaussian
smoothing of convex functions with Hölder-continuous sub-gradients;
this method applies under regularity conditions and convexity
assumptions on $g$ which are not implied by
\Cref{ass:smooth}-\Cref{ass:nonsmooth}.
In \cite{durmus:etal:2018}, the authors add a Moreau-Yosida envelope
term and a gradient term.
In \cite{luu:etal:2021}, $\stepalgo$ composes a Moreau-Yosida envelope
of $g(\A \cdot)$ and a gradient step.
Let us cite 
\cite{atchade:2015,schreck:etal:2016,salim:richtarik:2020} who also
use proximal operators in order to define trans-dimensional Monte
Carlo samplers -- an objective which is out of the scope defined by
A\ref{ass:smooth}. See also  \cite{pereyra:etal:survey:2015} and \cite[section 3]{luengo:etal:2020} for a survey on MCMC samplers using optimization techniques, and \cite[Section 4]{liu:liang:wong:2000} for optimization techniques combined with Multiple Try Metropolis strategies.

However, in the optimization context \eqref{eq:algoMPA} with a non-smooth objective function $F$, explicit sub-gradient method is suboptimal with rate $O(1/\sqrt{n})$ compared to proximal-based strategy such as forward-backward (relying on an implicit sub-gradient method) whose rate is  $O(1/n)$ \cite{nesterov:2003}. Proximal-based algorithms are thus especially adapted to handle non-smooth functions but also smooth functions that do not have a Lipschitz gradient. The only limitation of using proximal-based algorithms relies on the possible difficulty to derive a closed-form expression of the proximity operator. However, for most common functions encountered in signal processing, closed-forms exist and have been summarized on the website ProxRepository \url{http://proximity-operator.net/scalarfunctions.html} - see also \Cref{lem:propfoL} below, for standard functions composed with linear operators. When $F
= f + g$ where $f$ and $g$ satisfy Assumption~A\ref{ass:smooth},
deriving a proximal activation $\stepalgo$ is often a tedious task as
no closed form expression for $\prox_{f+g}$ exists in a general
framework
\cite{yu_decomposing_2013,Pustelnik_N_2017_j-ieee-spl}.  The standard solution consists in a
proximal-gradient activation: $\stepalgo(\paramvect) \eqdef
\prox_{\gamma g}(\paramvect - \gamma\nabla f(\paramvect))$. When
dealing with a blockwise structure for $g$ as in A\ref{ass:nonsmooth},
the choice of $\stepalgo$ has to manage both the additive structure of
$g$ and the combination of the $g_{i,j}$'s with a linear
operator. Unfortunately, the proximity operator has a closed form
expression in very limited cases recalled below in
Lemma~\ref{rem:linop}.

\begin{lemma} \label{lem:propfoL}Let $\A\in \mathbb{R}^{c\times d}$. 
\begin{enumerate}
\item \label{prop:compii} Let $g \eqdef \frac{1}{2}\Vert  \cdot - \modif{\boldsymbol{\mathsf{z}}} \Vert^2$ with $\boldsymbol{\mathsf{z}}\in
  \mathbb{R}^c$. For every $\gamma>0$,
$$
\prox_{\pas g(\A \cdot)} = (\gamma \A^\top \A + \Id_d)^{-1}( \cdot + \gamma \A^\top \modif{\boldsymbol{\mathsf{z}}}).
$$
\item 
\label{prop:compi}
Let $g$ be a proper lower semi-continuous convex
function. Let $\A \A^{\top}$ be invertible. For every $\gamma>0$,
$$\prox_{\pas g(\A \cdot)}=\Id_d- \A^{\top} (\A
\A^{\top})^{-1}(\Id_d-\prox_{\pas g}^{(\A \A^{\top})^{-1}})\A$$
where, for every $\modif{\boldsymbol{\mathsf{x}}} \in \mathbb{R}^d$,
$$
 \!\!\!\!\!\!\!\!\!\!  \prox_{\gamma g}^{(\A \A^{\top})^{-1}}\!\!(\modif{\boldsymbol{\mathsf{x}}})\!\!\eqdef\!\!\underset{\modif{\boldsymbol{\mathsf{y}}}}{\mathrm{argmin}}\; g(\modif{\boldsymbol{\mathsf{y}}}) + \frac{1}{2\gamma} (\modif{\boldsymbol{\mathsf{x}}}-\modif{\boldsymbol{\mathsf{y}}})^\top (\A \A^{\top})^{-1}(\modif{\boldsymbol{\mathsf{x}}}-\modif{\boldsymbol{\mathsf{y}}}). 
$$
\item 
  \label{prop:compiii}Let $g(\A\cdot) \eqdef \sum_{\ell=1}^c
g_\ell( \modif{\boldsymbol{\mathsf{a}}}_{\ell} \cdot)$ where $g_\ell$ is convex, lower
semi-continuous, and proper from $\mathbb{R}$ to $]-\infty,+\infty]$,
and $\boldsymbol{a}_{\ell}\in \RR^d$ denotes the row $\# \ell$ of
$\A$.  Suppose that $\A \A^\top=\modif{\boldsymbol{\Lambda}}$, where
$\modif{\boldsymbol{\Lambda}} \eqdef \mathrm{diag}(\chi_1, \ldots, \chi_c)$ and
$\chi_\ell>0$.  Then, for every $\gamma>0$: for all $\boldsymbol{\eta}\in \RR^{d}$  $$\!\!\!\!\!\!\!\!\!\!
\prox_{\pas
g(\A \cdot)} (\boldsymbol{\eta})
=\boldsymbol{\eta}-\A^{\top} \modif{\boldsymbol{\Lambda}}^{-1} \big(\A\boldsymbol{\eta}-\prox_{\pas \modif{\boldsymbol{\Lambda}}
g } (\A\boldsymbol{\eta}) \big) $$ where for all $\boldsymbol{\zeta}
= \boldsymbol{\zeta}_{1:c}$, we set $\modif{\boldsymbol{\Lambda}} g (\boldsymbol{\zeta}
) \eqdef \sum_{\ell=1}^c \chi_\ell g_\ell (\zeta_\ell)$.
\end{enumerate}
\label{rem:linop}
\end{lemma}
\begin{proof} 1) See \cite{Combettes2011}. 2) This result is a direct consequence of \cite[Proposition 23.25 (ii)]{bauschke2011convex} and \cite[Proposition 23.345(ii)-(iii)]{bauschke2011convex}. 3) Result extracted from~ \cite{Pustelnik_N_2011_ieee-tip_par_pai} and a direct consequence of 2) for specific choices of $\A$ and $g$.
\end{proof}

 \subsection{ Novel  Blockwise Proximal Monte Carlo Samplers: {\tt PGdec} and {\tt PGdual}}
 \label{sec:PGdecAndPGdual}
None of the algorithms recalled in the previous section directly applies
to the context of Assumptions A\ref{ass:smooth} to
A\ref{ass:nonsmooth}. We propose two novel Metropolis-Hastings algorithms: the
{\tt PGdec} sampler and the {\tt PGdual} sampler, which use the
  proximal operator to handle the nonsmooth part $g$ of $-\ln \pi$, its
  additive structure and its combination with linear operators. Such strategies can be used as building blocks for more general MCMC samplers such as Multiple Try based samplers (see e.g. \cite{martino:2018} and  \cite{bedard:douc:moulines:2012}) and more generally, adaptive Monte Carlo samplers (see e.g.  \cite{andrieu:thoms:2008}); 
An extension to Gibbs sampler is proposed and discussed in \Cref{sec:gibbssampler}.\\

\noindent$\bullet$ \noindent \textbf{ The {\tt PGdec} sampler.} Additionally to Assumptions~A\ref{ass:smooth} and
A\ref{ass:nonsmooth}, we further assume: 
\begin{assumption} \label{ass:nonsmooth3} Each function $g_{i,j}(\A_{i,j}\cdot)$ possesses a proximal operator  having a closed form expression.
\end{assumption}
 \Cref{ass:nonsmooth3}
assumes that each component $g_{i,j}(\A_{i,j} \cdot)$ has a tractable
proximal operator which does not imply that the  sum
$g = \sum_{j=1}^J \sum_{i=1}^{I_j} g_{i,j}$ admits a tractable proximal operator.

The
         {\it Proximal-Gradient Decomposition sampler} ({\tt PGdec})
         is described by \autoref{algo:PGD}.  It is a
         Metropolis-Hastings sampler with Gaussian proposal:
         conditionally to the current point $\paramvect^n$, {\tt
           PGdec} proposes a move to $\paramvect^{n+1/2} =
         (\paramvect_{1}^{n+1/2},\cdots, \paramvect_{J}^{n+1/2})$
         sampling independently the $J$ blocks from Gaussian
         distributions (see $\paramvect^{n+1/2}_j$ in
         \autoref{line:updateblock} of \autoref{algo:PGD}); Then an
         acceptance-rejection step is applied (see
         \eqref{eq:AR:PGD}).  The originality of our method is the
         definition of $\stepalgo$: for every $j \in  \{1, \ldots, J\}$
         and $i \in \lbrace 1, \ldots, I_j \rbrace$,
\begin{equation}\label{eq:drift:PGD}
 \stepalgo_{i,j}^{\textrm{PGdec}} (\paramvect) \eqdef \prox_{\pas_j \, g_{i,j}(\A_{i,j} \cdot)}\left(\paramvect_j - \pas_j \, \nabla_j f(\paramvect) \right),
\end{equation}
where $\pas_j$ is a positive step size and $\nabla_j$ denotes the
differential operator w.r.t. the block $\# j$ of $\paramvect$.  The
proposed drift takes benefit of the blockwise separable expression of
$g$ and computes at each iteration the proximal operator associated to
part of the sum in order to perform the proximal activation with a
closed form expression.  Conditionally to $\paramvect^n$, for each
block $\# j$, one of the component $\# i_j$ is selected at random in
$\{1, \ldots, I_j\}$ (see \autoref{line:selectblock}); then,
$\paramvect^{n+1/2}_j$ is sampled from a $\rset^{d_j}$-valued Gaussian
distribution with expectation $
\stepalgo_{i_j,j}^{\textrm{PGdec}}(\paramvect^n)$ and covariance
matrix $\C_{i_j,j}$. We denote by $q_{i,j}(\paramvect, \boldsymbol{\tau} )$ the
density of the distribution $\mathcal{N}_{d_j}( \stepalgo_{i,j}(\paramvect),
\C_{i,j})$ evaluated at $\boldsymbol{\tau}  \in \rset^{d_j}$:
\[
q_{i,j}(\paramvect,\boldsymbol{\tau} ) \eqdef \frac{\exp\left(-0.5 \, (\boldsymbol{\tau} 
  - \stepalgo_{i,j}(\paramvect))^\top \,
  \C_{i,j}^{-1} (\boldsymbol{\tau}  -
  \stepalgo_{i,j}(\paramvect))\right)}{\sqrt{2\pi}^{d_j}
  \sqrt{\mathrm{det}(\C_{i,j})}}.
\]

\begin{remark}
Let $j \in \{1, \ldots, J\}$ and $i \in \{1, \ldots, I_j \}$.  Given
$\paramvect^n= (\paramvect^n_{1:J}) \in \rset^{d}$,
$\stepalgo_{i,j}^{\textrm{PGdec}}(\paramvect^n)$ successively computes
a gradient step w.r.t. the smooth function $f$ and the variable
$\paramvect_j$, and a proximal step with respect to the function
$g_{i,j}(\A_{i,j} \cdot)$; the step size is $\gamma_j$ for both
steps. Hence, $\stepalgo_{i,j}^{\textrm{PGdec}}(\paramvect^n)$ is a
Proximal-Gradient (PG) step w.r.t. the function
\begin{multline}\label{eq:fonction:image}
  \paramvect_j \mapsto f(\paramvect_{1:j-1}^n, \paramvect_j,
  \paramvect_{j+1:J}^n)  + g_{i,j}(\A_{i,j} \paramvect_j).
\end{multline}
\end{remark}

 \begin{example}[\Cref{ex:PGDcase1} to follow] The proximal operator
of $\Rvect \mapsto
\lambdatime \|\D2 \Rvect \|_1$ is not explicit, so that, when decomposing $g$ as proposed in \Cref{ex:PGDcase1}, {\tt PGdec} can not be applied to  approximate  $\pi_\Zvect$.
\end{example}
\begin{example}[\Cref{ex:PGDcase2} to follow] \label{ex:PGDcase5}
For all $i \in \{1, 2,3 \}$, $\A_{i,1} \A_{i,1}^\top$ is the identity
matrix so that, by \Cref{lem:propfoL},
$\prox_{\pas_1\, \lambdatime \| \A_{i,1} \cdot \|_1}$ is explicit and
given by $(\Id_T - \A_{i,1}^\top \A_{i,1})
+ \A_{i,1}^\top \prox_{\pas_1 \, \lambdatime \| \cdot \|_1}(\A_{i,1} \cdot)$. In
addition, $\prox_{\pas_2 \, \lambdaO \|\cdot \|_1}$ has a closed form
expression. Hence, when decomposing $g$ as proposed
in \Cref{ex:PGDcase2}, {\tt PGdec} can be applied to  approximate
$\pi_\Zvect$.
\end{example}

\begin{algorithm}[t]
  \KwData{ $d_j \times d_j$ positive definite matrices $\C_{i,j}$;
    $\pas_j >0$; a positive integer $N_{\mathrm{max}}$;
    $\paramvect^0 \in \Dset$} \KwResult{A $\Dset$-valued sequence
    $\{\paramvect^n, n \in \{0, \ldots,
    N_{\mathrm{max}} \} \}$} \For{$n=0, \ldots,
    N_{\mathrm{max}}-1$}{ \For{$j = 1, \ldots, J$}{ Sample
    $i_j \in \lbrace 1, \ldots, I_j \rbrace$ with probability
    $1/I_j$ \label{line:selectblock} \; Sample
    $\boldsymbol{\xi}^{n+1}_{j} \sim \mathcal{N}_{d_j}(\0mat_{d_j}, \C_{i_j,j})$;\\ Set $\paramvect_j^{n+\frac{1}{2}}
    = \stepalgo_{i_j,j}(\paramvect^n)
    + \boldsymbol{\xi}^{n+1}_{j}$\; \label{line:updateblock} } \label{line:AR}
    Set $\paramvect^{n+1} = \paramvect^{n+\frac{1}{2}}$ with
    probability \begin{equation} \label{eq:AR:PGD}
    1 \wedge \frac{\pi(\paramvect^{n+\frac{1}{2}})}{\pi(\paramvect^n)}
\prod_{j = 1}^J\frac{q_{i_j,j}(\paramvect^{n+\frac{1}{2}}, \paramvect_j^n)}{q_{i_j,j}(\paramvect^n, \paramvect_j^{n+\frac{1}{2}})}
\end{equation}   and $\paramvect^{n+1} = \paramvect^n$ otherwise. 
  } \caption{ Blockwise Metropolis-Hastings samplers \label{algo:PGD}}
\end{algorithm}
\vspace{0.3cm}

\noindent $\bullet$ \noindent \textbf{\bf The {\tt PGdual}
  sampler.}  The {\it Proximal-Gradient dual sampler} ({\tt PGdual})
is defined along the same lines as \autoref{algo:PGD}. It is
  designed for situations when for any $i,j$, the dimensions of the
  matrices $\A_{i,j}$ satisfy $c_{i,j} \leq d_j$ and $\A_{i,j}$ can be
  augmented in an invertible $d_j \times d_j$ matrix -- denoted by
  $\barA_{i,j}$.

For every $j \in \{1, \ldots, J \}$ and $i \in \{1, \ldots, I_j\}$,
let $\barA_{i,j}$ be a $d_j \times d_j$ invertible matrix such that
for any $\paramvect_j \in \rset^{d_j}$,
$( \barA_{i,j} \paramvect_j)_{d_j -c_{i,j}+1:d_j}
= \A_{i,j} \paramvect_j$; For $\mathbf{x} = (x_1,\ldots, x_{d_j})\in
\RR^{d_j}$, define $\bar{g}_{i,j}(\mathbf{x} ) \eqdef
g_{i,j}(\modif{\mathbf{x}}_{d_j-c_{i,j}+1:d_j})$. {\tt PGdual} uses the drift
functions
\begin{equation} \label{eq:driftPGdual0}
 \stepalgo_{i,j}^{\textrm{PGdual}} (\paramvect) \eqdef   \barA_{i,j}^{-1} \prox_{\pas_j \bar{g}_{i,j}}\left(\barA_{i,j} \paramvect_j - \pas_j \barA_{i,j}^{-\top} \nabla_j f(\paramvect) \right).
\end{equation}

\begin{remark}  \label{rem:PGdual}
For every $j\in\{1,\ldots, J\}$, select $i_j \in \{1,\ldots, I_j\}$,
and consider the partial objective function:
\begin{equation}
\paramvect\mapsto f(\paramvect_1,\ldots, \paramvect_J) + \sum_j  \bar{g}_{i_j,j}(\barA_{i_j,j}\paramvect_j).
\end{equation}
Consider the one-to-one maps $\widetilde{\paramvect}_{j} \eqdef
\barA_{i_j,j}\paramvect_j$ for any $j$, and the application
\begin{equation}
\widetilde{\paramvect} = \widetilde{\paramvect}_{1:J} \mapsto
f(\barA_{i_1,1}^{-1} \widetilde{\paramvect}_{1}, \ldots,
\barA_{i_J,J}^{-1} \widetilde{\paramvect}_{J}) + \sum_j
\bar{g}_{i_j,j}( \widetilde{\paramvect}_{j}).
\end{equation}
The PG step w.r.t. $\widetilde{\paramvect}_{j}$ reads:
\begin{multline}
\prox_{\pas_j \, \bar{g}_{\modif{i_j},j}}\left(\widetilde{\paramvect}_{j} -
\gamma_j \, \barA_{i_j,j}^{-\top} \, \nabla_j f(\barA_{i_1,1}^{-1}
\widetilde{\paramvect}_{1}, \ldots, \barA_{i_J,J}^{-1}
\widetilde{\paramvect}_{J}) \right)\\ = \prox_{\pas_j \,
  \bar{g}_{\modif{i_j},j}}\left(\barA_{i_j,j} {\paramvect}_{j} - \gamma_j \,
\barA_{i_j,j}^{-\top} \, \nabla_j f(\paramvect) \right).
\end{multline}
Therefore, applying a PG step w.r.t. the variable
$\widetilde{\paramvect}_{j}$ in this "{\it dual}" space, and going back
to the original space by applying $\barA_{i_j,j}^{-1}$ leads to the
drift $\stepalgo_{i_j,j}^{\textrm{PGdual}}$ in
\eqref{eq:driftPGdual0}. \\
\modif{From \eqref{eq:driftPGdual0} and \Cref{lem:propfoL} \Cref{prop:compi}, we have
\begin{equation}\label{eq:VMFB}
 \stepalgo_{i,j}^{\textrm{PGdual}} (\paramvect) \eqdef    \prox_{\pas_j \bar{g}_{i,j}(\barA_{i,j} \cdot)}^{\barA_{i,j}^\top \barA_{i,j}}\left(\paramvect_j - \pas_j (\barA_{i,j}^\top \barA_{i,j})^{-1} \nabla_j f(\paramvect) \right)
\end{equation}
thus showing that it is a {\it Variable Metric Proximal-Gradient}
step in the metric induced by $\barA_{i,j}^\top \barA_{i,j}$, for the
minimization of $\paramvect_j \mapsto f(\paramvect) + \bar{g}_{i,j}(\barA_{i,j} \paramvect_j)$. See \cite{chen:rockafeller:1997,combettes:vu:2014}; see also \cite[section 2]{abry:etal:2023} in the specific case $J=1,I=1$.}
\end{remark}

\begin{example}[\Cref{ex:PGDcase1} to follow] \label{ex:PGDcase3}
Denote by $\barDinv$ the $T \times T$ invertible  matrix obtained by augmenting $\D2$ with two vectors in $\rset^T$ as follows:
\begin{equation} \label{eq:def:barD2}
\barDinv \eqdef \begin{bmatrix}
1 & 0 & 0 & \ldots & 0 \\ -2/\sqrt{5} & 1/\sqrt{5} & 0 & \ldots & 0 \\
& & \D2 & & \\
\end{bmatrix};
\end{equation}
$\D2$ is subdiagonal with normalized rows (see Eq.~\eqref{eq:def:D2}) and we define $\barDinv$  so that it has the same properties. Then $
(\barDinv \Rvect)_{3:T} = \D2 \Rvect$, and for $\Rvect \in \rset^T$,
$\bar{g}_{1,1}(\Rvect) = \lambdatime \| \Rvect_{3:T} \|_1$. In
addition, $\bar{g}_{1,2} = g_{1,2}$. Hence, when decomposing $g$ as
proposed in \Cref{ex:PGDcase1}, {\tt PGdual} can be used to approximate
$\pi_\Zvect$.
\end{example}
\noindent $\bullet$ \textbf{Choice of the covariance matrix $\C$}.
Based on \Cref{rem:PGdual},
  $\barA_{i,j} \stepalgo^{\textrm{PGdual}}_{i,j}(\widetilde{\paramvect})$
  is a PG step in a dual space. Since such a step can be seen as an
  extension of a gradient step to a nonsmooth function, a natural idea
  is to mime the MALA proposal and add a Gaussian
  noise \modif{$\widetilde{\boldsymbol{\xi}} $} with covariance matrix
  $2 \pas_j \Id_{d_j}$ in the dual space. Therefore, in the original
  space, the noise \modif{$\boldsymbol{\xi}
  :=\barA_{i,j}^{-1} \widetilde{\boldsymbol{\xi}}$ has covariance
  matrix} $\C_{i,j} \eqdef
  2 \pas_j \barA_{i,j}^{-1} \barA_{i,j}^{-\top}$. In \eqref{eq:driftPGdual0},
  note the preconditioning matrix $\barA_{i,j}^{-\top}$ before the
  gradient and the matrix $\barA_{i,j}^{-1}$ before the proximal
  operator: there is a parallel between such a choice of $\C_{i,j}$
  and the scaled MALA proposal
  mechanism \eqref{eq:LangevinT:dynamic}. \\ \modif{Equivalently, note
  the parallel between the preconditioning of the gradient term
  in \eqref{eq:VMFB}, the covariance matrix $\C_{i,j}$
  and \eqref{eq:LangevinT:dynamic}.}
  
  An equivalent argumentation can be developed for $\stepalgo^{\textrm{PGdec}}$.

\subsection{Interpretations of {\tt PGdec} and {\tt PGdual}}
In the case $\A_{i,j} \A_{i,j}^\top = \nu_{i,j} \Id_{c_{i,j}}$ with
$\nu_{i,j} >0$, we have from Lemma~\ref{lem:propfoL} \Cref{prop:compiii},
\begin{multline} \label{eq:decomposition:mu} \stepalgo_{i,j}^{\textrm{PGdec}} (\paramvect^n) =
  \left( \Id_{d_j} - \modif{\boldsymbol{\Pi}}_{i,j}\right) \modif{\boldsymbol{\mathsf{G}}}_{j}(\paramvect^n) \\
  + \A_{i,j}^\top \left(\A_{i,j} \A_{i,j}^\top \right)^{-1} \prox_{ \nu_{i,j} \pas_{j}
  g_{i,j}} \left( \A_{i,j} \modif{\boldsymbol{\mathsf{G}}}_{j}(\paramvect^n) \right),
\end{multline}
where $\modif{\boldsymbol{\mathsf{G}}}_{j}(\paramvect^n) \eqdef \paramvect_j^n - \pas_j \nabla_j
f(\paramvect^n)$ is a gradient step, and
\[
\modif{\boldsymbol{\Pi}}_{i,j} \eqdef \A_{i,j}^\top \left(\A_{i,j} \A_{i,j}^\top
\right)^{-1} \A_{i,j}
\]
 is the orthogonal projection matrix on the range of
 $\A_{i,j}^\top$. \eqref{eq:decomposition:mu} shows that
 $\stepalgo_{i,j}^{\textrm{PGdec}} (\paramvect^n)$ is the sum of two
 orthogonal terms: the first one is the orthogonal projection of $\modif{\boldsymbol{\mathsf{G}}}_{j}(\paramvect^n)$ on the orthogonal space of the range of
 $\A_{i,j}^\top$, and the second term is in the range space of
 $\A_{i,j}^\top$.  This second term may be seen as a {\it
 proximal-contraction} of $\modif{\boldsymbol{\Pi}}_{i,j} \modif{\boldsymbol{\mathsf{G}}}_{j}(\paramvect^n)$.

 \autoref{theo:driftdual} makes the {\tt PGdual} idea explicit by
 proposing a matrix $\barA_{i,j}$ augmenting $\A_{i,j}$, computing the
 associated drift $\stepalgo_{i,j}^{\mathrm{PGdual}}$ and comparing it
 to $\stepalgo_{i,j}^{\textrm{PGdec}}$.

\begin{theorem} \label{theo:driftdual} Assume A\autoref{ass:smooth}
  and A\autoref{ass:nonsmooth}. Let $j \in \{1, \ldots,
  J \}$ and $i \in \{1, \ldots, I_j \}$.  Assume that $c_{i,j} < d_j$
  and $\A_{i,j} \A^\top_{i,j} $ is invertible.  \\ Let $\U_{i,j}$ be a
  $(d_j-c_{i,j}) \times d_j$ matrix such that $\U_{i,j} \A^\top_{i,j}=
  \0mat_{ (d_j-c_{i,j}) \times c_{i,j}}$ and $\U_{i,j} \U_{i,j}^\top$ is
  invertible.  \\ Then, the matrix 
  $\barA_{i,j} \eqdef [\U_{i,j}; \A_{i,j}]\in \mathbb{R}^{d_j\times
  d_j}$ is invertible.  For any $\paramvect \in \Dset$,
  $\stepalgo_{i,j}^{\textrm{PGdual}}(\paramvect)$ given
  by \eqref{eq:driftPGdual0} is equal to
\begin{align}
\label{eq:pgdual-pgd}
& \A_{i,j}^\top \left(\A_{i,j} \A_{i,j}^\top \right)^{-1}
\prox_{\pas_j g_{i,j}} \left( \A_{i,j} \left( \paramvect_j - \pas_j
\widetilde{\modif{\boldsymbol{\Omega}}}_{i,j} \nabla_j f(\paramvect)\right) \right)
\nonumber\\ & \quad + (\Id_{d_{j}} - \modif{\boldsymbol{\Pi}}_{i,j}) \left(\paramvect_j -
\pas_j \modif{\boldsymbol{\Omega}}_{i,j} \nabla_j f(\paramvect) \right),
  \end{align}
where
\[
\modif{\boldsymbol{\Omega}}_{i,j} \eqdef \U^\top_{i,j} (\U_{i,j} \U^\top_{i,j})^{-2}
\U_{i,j}, \ \  \widetilde{\modif{\boldsymbol{\Omega}}}_{i,j} \eqdef \A^\top_{i,j} (\A_{i,j}
\A^\top_{i,j})^{-2} \A_{i,j}.
\]
Additionally, when $\U_{i,j} \U_{i,j}^\top = \Id_{d_j-c_{i,j}}$, $\A_{i,j}
\A_{i,j}^\top = \Id_{c_{i,j}}$, and under A\autoref{ass:nonsmooth3} then $\stepalgo_{i,j}^{\textrm{PGdual}} = \stepalgo_{i,j}^{\textrm{PGdec}}$ where
$\stepalgo_{i,j}^{\textrm{PGdec}}$ is given
by \eqref{eq:decomposition:mu}.
\end{theorem}
\begin{proof}  The main ingredients are the equalities
\[
\barA^{-1}_{i,j} = \begin{bmatrix} \U^\top_{i,j} (\U_{i,j}
  \U^\top_{i,j})^{-1} & \A^\top_{i,j} (\A_{i,j}
  \A^\top_{i,j})^{-1} \end{bmatrix};
\] 
and $\U^\top_{i,j}(\U_{i,j} \U^\top_{i,j})^{-1} \U_{i,j} + \modif{\boldsymbol{\Pi}}_{i,j} =
\Id_{d_j}$.  The proof follows from standard matrix algebra. A
detailed proof is given in \autoref{proof:PGD} of the Supplementary
material.
\end{proof}
The result remains true when $c_{i,j} = d_j$; in that case, $\barA_{i,j} =
\A_{i,j}$, $\Id_{d_j} - \modif{\boldsymbol{\Pi}}_{i,j} = \0mat_{d_j \times d_j}$ and
$\widetilde{\modif{\boldsymbol{\Omega}}}_{i,j} =
\A_{i,j}^{-1}\A_{i,j}^{-\top}$. \autoref{theo:driftdual} provides
sufficient conditions for the drift function
$\stepalgo_{i,j}^{\textrm{PGdec}}$ and the drift function
$\stepalgo_{i,j}^{\textrm{PGdual}}$ to be equal. Observe that the
conditions on $\U_{i,j}$ are satisfied as soon as the rows of
$\U_{i,j}$ are orthonormal and orthogonal to the rows of $\A_{i,j}$.

Let us derive two strategies for the application of {\tt PGdual} to
sample the target density defined
by \eqref{eq:target:covid}.
\begin{example}[\Cref{ex:PGDcase1} and \Cref{ex:PGDcase3}, to follow]  \label{ex:PGDcase4}
A first strategy is to decompose $g$ as
in \Cref{ex:PGDcase1}; \Cref{ex:PGDcase3} provides a possible 
augmentation of $\D2$. Another one is proposed
by \autoref{theo:driftdual}:
\begin{equation}\label{eq:barDo}
  \barDo \eqdef \left[\begin{matrix} \U_{1,1} \\ \D2 \end{matrix} \right] \in \rset^{T \times
    T} \end{equation} where $\U_{1,1}$ is obtained by making
    orthogonal the first two rows of $\barDinv$ \modif{(see \eqref{eq:def:barD2})} and making them
    orthogonal to the rows of $\D2$. 
    \modif{With this choice, $\barDo$ has a condition number one order of magnitude smaller compared to $\barDinv$, which is likely to impact the convergence of Markov chains (see Section~\ref{sec:results}).}\\ This strategy acts globally on
    $\ln
\pi_\Zvect$ by proposing, at each iteration, a PG
approach for the function $f_\Zvect +\lambdatime \|\D2\cdot\|_1 +
\lambdaO \| \cdot \|_1$. It will be numerically
explored in \autoref{sec:assessment}.
\end{example}
\begin{example}[\Cref{ex:PGDcase2} to follow] 
A second strategy is to decompose $g$ as in \Cref{ex:PGDcase2}, and
define the matrices $\barA_{i,j}$ as described in
\autoref{theo:driftdual}. This strategy defines
$\stepalgo^{\mathrm{PGdual}}$ by considering part of $\ln \pi_\Zvect$: at each iteration, having selected $i_1 \in \{1, 2,
3\}$, it proposes a PG approach for the function $f_\Zvect +
\lambdatime \|\A_{i_1,1}\cdot\|_1 + \lambdaO \| \cdot \|_1$ (see
\Cref{rem:PGdual}).
\end{example}

\medskip

\subsection{ Convergence analysis of {\tt PGdec} and {\tt PGdual}.}
We prove that both {\tt PGdec} and {\tt PGdual} produce a sequence of
points $\{\paramvect^n,n \geq 0\}$ which is an ergodic Markov chain
having $\pi$ as its unique invariant distribution.
\begin{proposition}\label{prop:MCMCsampler}
  Assume A\autoref{ass:smooth}, A\autoref{ass:nonsmooth} and
  A\autoref{ass:nonsmooth3}. Assume also that $\pi$ is continuous on
  $\Dset$. Then the sequence $\{\paramvect^n, n \geq 0 \}$ given
  by \autoref{algo:PGD} applied with $\stepalgo
  = \stepalgo^{\mathrm{PGdec}}$ is a Markov chain, taking values in
  $\Dset$ and with unique invariant distribution $\pi$. In addition,
  for any initial point $\paramvect^0 \in \Dset$ and any measurable
  function $h$ such that $\int
  |h(\paramvect)| \, \pi(\paramvect) \rmd \paramvect
  < \infty$, \[ \lim_{N \to \infty} \frac{1}{N} \sum_{n=1}^N
  h(\paramvect^n) = \int
  h(\paramvect) \, \pi(\rmd \paramvect), \quad \text{with probability
  one.}  \] The same result holds when $\stepalgo
  = \stepalgo^{\mathrm{PGdual}}$.
\end{proposition}
\begin{proof}
 We have $\paramvect^{n+1} \in \Dset$ since $\pi(\paramvect^{n+1/2})
 =0$ when $\paramvect^{n+1/2} \notin \Dset$.  The other properties
 result from \cite[Lemmas 1.1. and 1.2]{mengersen:tweedie:1996}
 and \cite[Propositions 10.1.1 and 10.4.4 and Theorem
 17.0.1]{meyn:tweedie:2009}. The Harris recurrence property (required
 for the law of large numbers to hold for any $\paramvect^0$) can be
 proved along the same lines as \cite[Theorem
 6(v)]{roberts:rosenthal:2006}. See the detailed proof
 in \autoref{proof:PGD} of the Supplementary material.
\end{proof}
\begin{remark} 
A slight adaption of the proof shows that when selecting the indices
$i_j$, the uniform distribution on $\{1, \ldots, I_j\}$ can be
replaced with a probability distribution $\{\rho_{i,j}(\paramvect^n),
i=1, \ldots, I_j \}$ depending on the current value $\paramvect^n$, up
to a modification of the acceptance-rejection ratio \eqref{eq:AR:PGD}
(see \autoref{proof:PGD} of the Supplementary material).
\end{remark}
\medskip

\subsection{ {\tt Gibbs PGdec} and {\tt Gibbs PGdual} samplers.}
\label{sec:gibbssampler}
The additive structure of $g$ in \Cref{ass:nonsmooth}, naturally
suggests the extension of {\tt PGdec} and {\tt PGdual} to the Gibbs
sampler algorithm~\modif{\cite{geman:geman:1984,gelfand:smith:1990}}.  Since whatever $\paramvect
= \paramvect_{1:J} \in \rset^d$, exact sampling from the conditional
distributions on $\rset^{d_j}$
\begin{multline*}
{\boldsymbol \tau} \mapsto \pi_j({\boldsymbol \tau} \vert \paramvect) \propto
\exp(-f(\paramvect_{1:j-1},{\boldsymbol \tau},\paramvect_{j+1:J}) \\   - \sum_{i=1}^{I_j}
g_{i,j}(\A_{i,j} {\boldsymbol \tau}) ) \1_{\Dset}(\paramvect_{1:j-1},{\boldsymbol \tau},\paramvect_{j+1:J})
\end{multline*}
for $j=1, \ldots, J$, is not always explicit, we propose a
Metropolis-within-Gibbs strategy; see \cite{chan:geyer:1994}.  The
pseudo-code of the so-called {\tt Gibbs Blockwise Proximal} sampler is given
by \autoref{algo:gibbsPGD} in the case of a systematic scan order of
the $J$ components (our method easily extends to other scan orders;
details are left to the reader): at each iteration $\# (n+1)$, and for
each block $\# j$, \textit{(i)} sample at random $i_j \in \{1, \ldots,
I_j\}$, \textit{(ii)} sample a candidate from
$\mathcal{N}_{d_j}(\stepalgo_{i_j,j}(\varthetavect), \C_{i_j,j})$ where
$\varthetavect \eqdef (\paramvect^{n+1}_{1:j-1}, \paramvect^n_{j:J}) \in \rset^d$
is the current value of the chain, and \textit{(iii)} accept/reject
this candidate via a Metropolis step targeting the distribution
$\pi_j(\cdot \vert \varthetavect)$.
\begin{algorithm}[htbp]
  \KwData{ $d_j \times d_j$ positive definite matrices $\C_{i,j}$;
    $\pas_j >0$; a positive integer $N_{\mathrm{max}}$; $\param^0 \in
    \Dset$} \KwResult{A $\Dset$-valued sequence $\{\paramvect^n, n \in
    \{0, \ldots, N_{\mathrm{max}} \} \}$} Set $\varthetavect = \paramvect^0$
  \; \For{$n=0, \ldots, N_{\mathrm{max}}-1$}{ \For{$j = 1, \ldots,
      J$}{ Sample $i_j \in \lbrace 1, \ldots, I_j \rbrace$ with
      probability $1/I_j$\; Sample $\boldsymbol{\xi}^{n+1}_{j} \sim
      \mathcal{N}_{d_j}(\0mat_{d_j}, \C_{i_j,j})$ \; Set
      $\paramvect_j^{n+\frac{1}{2}} = \stepalgo_{i_j,j}(\varthetavect) +
      \boldsymbol{\xi}^{n+1}_{j}$\; Set $\paramvect^{n+1}_j =
      \paramvect^{n+\frac{1}{2}}_j$ with probability \label{line:AR}
 \[
1 \wedge \frac{\pi_j(\paramvect^{n+\frac{1}{2}}_j \vert \varthetavect)}{\pi_j(\paramvect^n_j \vert \varthetavect)}
\frac{q_{i_j,j}((\varthetavect_{1:j-1},\paramvect^{n+\frac{1}{2}}_j, \varthetavect_{j+1:J}), \paramvect_j^n)}{q_{i_j,j}(\varthetavect, \paramvect_j^{n+\frac{1}{2}})}
\]    and $\paramvect^{n+1}_j = \paramvect^n_j$ otherwise \;
    Update $\varthetavect = (\paramvect^{n+1}_{1:j}, \paramvect^{n}_{j+1:J})$} 
  }
  \caption{ Gibbs  Blockwise Proximal sampler \label{algo:gibbsPGD}}
\end{algorithm}

The {\tt Gibbs PGdec} and the {\tt Gibbs PGdual} samplers correspond
to \autoref{algo:gibbsPGD} applied resp. with $\stepalgo_{i,j} =
\stepalgo_{i,j}^{\mathrm{PGdec}}$ and $\stepalgo_{i,j} =
\stepalgo_{i,j}^{\mathrm{PGdual}}$.

\section{Covid-19 data}
\label{sec:data}
 To illustrate, assess, and compare the relevance and performance of the
Monte Carlo procedures for credibility interval estimation
proposed here, use is made of real
Covid-19 data made available by the \emph{Johns Hopkins
University}\footnote{
\href{https://coronavirus.jhu.edu/}{https://coronavirus.jhu.edu/}}. 
The
repository\footnote{\href{https://raw.githubusercontent.com/CSSEGISandData/COVID-19/master/csse_covid_19_data/csse_covid_19_data/csse_covid_19_time_series/}
{https://raw.githubusercontent.com/CSSEGISandData/COVID-19/master/csse$\_$covid$\_$19$\_$time$\_$series/}}
collects, impressively since the outbreak of the pandemic, 
daily new infections and new death counts from the National
Public Health Authorities of 200+ countries and territories of the
world.  
Counts are updated on a daily basis since the early stage of
the pandemic (Jan. 1st, 2020) until today.
This repository thus provides researchers with a remarkable dataset to analyze the pandemic. 

As mentioned in the Introduction section, because of the sanitary crisis context, data made available by most National
Public Health Authorities in the world are of very limited quality as they
are corrupted by outliers and missing or negative counts.
Data quality also varies a lot across countries or even within a given country depending on the phases and periods of the pandemic.

The present work uses the new infection counts only, as the estimation
of the space-time evolution of the pandemic reproduction number $\R$ is
targeted.

A mild and non-informative preprocessing is applied to data by replacing negative counts by a null value.

 \section{Monte Carlo sampler assessments}
\label{sec:assessment}
This section aims to assess and compare the performance for
several variations of the samplers {\tt PGdec} and {\tt PGdual}
introduced in \autoref{sec:sampler}, using the real Covid19 data
described in \autoref{sec:data}. The target distribution is $\pi_\Zvect$, given by \eqref{eq:target:covid}.
\vskip-2mm

\subsection{Assessment set-up} 
\label{sec:setup}
Performance
assessments were conducted on data from several countries and for
different periods of interests.
For space reasons, they are reported only for the United Kingdom and for a recent time
period (Dec. 6th, 2021 to Jan. 9th, 2022), with
corresponding $\Zvect$ in Fig.~\ref{fig:CIsbb}[top right plot].

Following \cite{Pascal2021}, we set $\lambdaO = 0.05$ and $\lambdatime
= 3.5 \times \sqrt{6} \, \sigma_\Zvect/4$ with $\sigma_\Zvect$ the
standard deviation of $\Zvect$.  Samplers are run with $N_\mathrm{max}
= 1e7$ iterations, including a {\it burn-in phase} of $3e6$
iterations.  Except for the plots in
Fig.~\ref{fig:asymptotique:UK}[first row], reported performance are
computed by discarding the points produced during the burn-in phase.
The initial values $\Z_{-\tau_\phi+1}, \ldots, \ \Z_{-1}, \Z_0$ are set
to the observed counts. For each algorithm, performances are computed
from averages over 15 independent runs.  Initial points consist of
random perturbations around the non-informative point $\Rvect^0 \eqdef
(1, \ldots, 1)^\top \in \rset^T$ and $\Ovect^0 \eqdef (0, \ldots,
0)^\top \in \rset^T$.  For all samplers, $\paramvect$ is seen as a
two-block vectors $(\paramvect_1=\Rvect, \paramvect_2=\Ovect)$; hence $J=2$. For {\tt PGdec}, $g$ is decomposed as in \Cref{ex:PGDcase2};
for {\tt PGdual}, it  is decomposed as in \Cref{ex:PGDcase1}.

\subsection{Samplers}
 \label{ssec:samplers}
\noindent  {\bf {\tt PGdec} and {\tt Gibbs PGdec}.}  These two samplers are run by
decomposing $g$ as described in \Cref{ex:PGDcase2}. This yields two
different drift functions for the blocks $\Rvect$ and $\Ovect$;
see \Cref{ex:PGDcase5}. For the $\Rvect$-part, the drift functions are
defined, for $i=1,2,3$, as:
\begin{multline}
\stepalgo_{i,1}^{\mathrm{PGdec}} (\paramvect) \eqdef   \left( \I_T - \A_{i,1}^\top \A_{i,1}
\right) \left( \Rvect - \pas_1 \nabla_1 f_\Zvect(\paramvect) \right) \\ \,  +
\A_{i,1}^\top \, \prox_{\pas_1 \lambdatime \| \cdot \|_1} \left(\A_{i,1}
\Rvect - \pas_1 \A_{i,1} \nabla_1 f_\Zvect(\paramvect) \right), \label{eq:driftPGD:R:covid}
\end{multline}
{obtained from \eqref{eq:decomposition:mu} with
$\A_{i,1} \A_{i,1}^\top$ the identity matrix}. For the
$\Ovect$-part, 
\begin{equation}\label{eq:driftPGD:O:covid}
\stepalgo_{1,2}^{\mathrm{PGdec}}(\paramvect) \eqdef \prox_{\pas_2 \lambdaO \| \cdot \|_1}\left(
\Ovect - \pas_2 \nabla_2 f_\Zvect(\paramvect)\right).
\end{equation}
$\nabla_1$ (resp. $\nabla_2$) denotes the gradient
operator w.r.t $\Rvect$ (resp. $\Ovect$).  

\noindent  {\bf {\tt PGdual} and {\tt Gibbs PGdual}.}  These two samplers are run by decomposing $g$ as described in
\Cref{ex:PGDcase1}. Consequently, for the $\Rvect$-part, the drift
function is given by
\begin{equation}\label{eq:driftPGdual:R:covid}
\stepalgo_{1,1}^{\mathrm{PGdual}}(\paramvect) \eqdef \barD2^{-1} \, \prox_{\pas_1 \lambdatime
  \| (\cdot)_{3:T} \|_1 } \left( \barD2 \Rvect - \pas_1  \barD2^{-\top} \nabla_1 f_\Zvect(\paramvect)\right)
\end{equation}
where $\barD2 \eqdef \barDinv$ (see \Cref{ex:PGDcase3}, the sampler is
  referred to as {\tt PGdual Invert (I)}) or $\barD2 \eqdef \barDo$
  (see \Cref{ex:PGDcase4}, the sampler is referred to as {\tt PGdual
  Ortho (O)}); for the $\Ovect$-part, for both {\tt PGdual I} and {\tt
  PGdual O},
\begin{equation}\label{eq:driftPGdual:O:covid}
\stepalgo_{1,2}^{\mathrm{PGdual}}(\paramvect) \eqdef \prox_{\pas_2 \lambdaO \| \cdot \|_1 }
\left(\Ovect - \pas_2 \nabla_2 f_\Zvect(\paramvect)\right) \eqsp.
\end{equation}

\noindent  {\bf RW and Gibbs RW.} 
For comparisons, we also run Random Walk-based samplers ({\tt RW})
with Gaussian proposals: {\tt RW} and {\tt Gibbs RW} are defined
respectively from \autoref{algo:PGD} and \autoref{algo:gibbsPGD}, with
$I_1 = I_2 = 1$ and $\stepalgo_{1,j}(\paramvect) = \paramvect$ for
$j=1,2$. 

\noindent  {\bf Covariance matrices $\C_{i,j}$.}
For the $\Rvect$-part (block $j=1$), based on the comment in \autoref{sec:PGdecAndPGdual}, we choose $\C_{1,1} \eqdef 2 \pas_1 \barDinv^{-1}
\barDinv^{-\top}$ for {\tt PGdual I} and {\tt Gibbs PGdual I}; and
$\C_{1,1} \eqdef 2 \pas_1 \barDo^{-1} \barDo^{-\top}$ for {\tt PGdual
  O} and {\tt Gibbs PGdual O}. The {\tt PGdec}-based samplers and the
  {\tt RW}-based samplers are run with the same covariance
  matrices. For {\tt PGdec}, this yields $\C_{i,1} \eqdef
  2 \pas_1 \barDinv^{-1} \barDinv^{-\top}$ for {\tt PGdec I}, {\tt
  Gibbs PGdec I} and $\C_{i,1} \eqdef
  2 \pas_1 \barDo^{-1} \barDo^{-\top}$ for {\tt PGdec O}, {\tt Gibbs
  PGdec O} for the three cases $i=1,2,3$. For {\tt RW},  $\C_{1,1} \eqdef 2 \pas_1 \barDinv^{-1}
\barDinv^{-\top}$ for {\tt RW I} and {\tt Gibbs RW I}; and
$\C_{1,1} \eqdef 2 \pas_1 \barDo^{-1} \barDo^{-\top}$ for {\tt RW O}
  and {\tt Gibbs RW  O}. For {\tt PGdec} and {\tt PGdual}, $\pas_1$ is the same step size as
  in \eqref{eq:driftPGD:R:covid} and
\eqref{eq:driftPGdual:R:covid}. For the $\Ovect$-part (block $j=2$), we choose $\C_{1,2} \eqdef 2\pas_2 \I_T$;  for {\tt PGdec} and {\tt PGdual}, $\pas_2$ is the same step size as in \eqref{eq:driftPGD:O:covid} and
\eqref{eq:driftPGdual:O:covid}.
  
\noindent  {\bf Twelve sampling strategies.} The present section will compare twelve different samplers, 
constructed from the three drift  strategies ({\tt PGdec,
PGDual, RW}), times two families ({\tt Metropolis-Hastings, Gibbs}),
times two choices for the covariance matrix $\C_{i,1}$ ({\tt I, O}).

\noindent  {\bf Step sizes.} 
Different strategies are compared for the definition of the step sizes
$(\pas_1, \pas_2)$. All of them consist in adapting the step sizes
during the burn-in phase in such a way that the mean
acceptance-rejection rate reaches approximately $0.25$ (which is
known to be optimal for some {\tt RW} Metropolis-Hastings samplers,
\cite{gelman:gilks:roberts:1997}). We observed that convergence occurs
before $5e5$ iterations, see Fig.~\ref{fig:asymptotique:UK}[row 1,
columns 2 and 3];  the  proposed samples are all rejected and the chain
does not move from its initial point during the first
iterations, when the step size is too large.  At the end of the
burn-in phase, the step sizes are frozen and no longer adapted. For the
Metropolis-Hastings samplers {\tt PGdec}, {\tt PGdual}, and {\tt RW},
$\pas_1$ is adapted and $\pas_2 \eqdef \pas_1$ (case {\tt PGdec}) of $\pas_2 \eqdef (\lambdatime/\lambdaO)^2 \pas_1$ (cases {\tt PGdual}, {\tt RW}). For
the Gibbs samplers, the acceptance-rejection steps are specific to
each block $\Rvect$ and $\Ovect$.  A first consequence is that a move
on one block can be accepted while the other one is not; this may
yield larger step sizes $(\pas_1, \pas_2)$ which in turn favor larger
moves of the chains. A second consequence is that we use the
acceptance-rejection rate for the $\Rvect$-part (resp. for the
$\Ovect$-part) to adapt $\pas_1$ (resp. $\pas_2$): the two step sizes
have their own efficiency criterion.

\subsection{Performance assessment criteria}
The sampler performances are assessed and compared using three
different criteria, see Fig.~\ref{fig:asymptotique:UK}.

\noindent  {\bf Distance to the MAP.} 
We compute the normalized distance $\|
\Rvect^n - \widehat{\Rvect}_{\mathrm{MAP}} \| /
\|\widehat{\Rvect}_{\mathrm{MAP}} \|$ where $\widehat{\Rvect}_{\mathrm{MAP}}$
denotes the MAP estimator, ({computed as in
\cite{Pascal2021}}): it is displayed vs. the iteration index $n$,  in the burn-in phase (row 1) and  after the burn-in phase (row
2).  This criterion quantifies the ability of the chains to perform a
relevant exploration of the distribution: The chains have to visit the
support of $\pi_\Zvect$, they show better ergodicity
rates when they are able to rapidly escape from low density regions to
move to higher density regions.  Paths start from a
non-informative initial point, considered as a point in a low density
region.  
This criterion permits to quantify a relevant behavior
of the Markov chain when
\textit{(i)} it drifts rapidly towards zero during the
burn-in phase and \textit{(ii)} it fluctuates in a large neighborhood
of zero after the burn-in phase.

\noindent  {\bf Autocorrelation function (ACF).} 
We then compute the ACF for each of the $2T$ components of the vector
$\paramvect$ along the Markov path, for a lag from $1$ to $1e5$. On
row 3, we report the mean value, over these $2T$ components, of the
absolute value of the ACF versus the lag. This criterion 
quantifies a relevant behavior of the Markov chain when
 it converges rapidly to zero; it is indeed related to
the effective sample size of a Markov chain (see
\cite{robert:casella:2005}) and to the mixing properties of the chain
(see e.g. \cite[Theorem 17.4.4 and Section
  17.4.3.]{meyn:tweedie:2009}). For example, a weaker ACF means that
  less iterations of the sampler are required to reach a given
  estimation accuracy.
  
\noindent  {\bf Gelman-Rubin (GR) statistic.} 
Finally, we compute the GR statistic (see
\cite{gelman:rubin:1992,brooks:gelman:1998}) which quantifies a
relevant behavior of the Markov chain when it converges rapidly to
one. It measures how the sampler forgets its initial value and
provides homogeneous approximations of the target distribution
$\pi_\Zvect$ after a given number of iterations. On row
4, we report this statistic versus the iteration index.

\begin{figure}[tbp]
\includegraphics[width=1\linewidth]{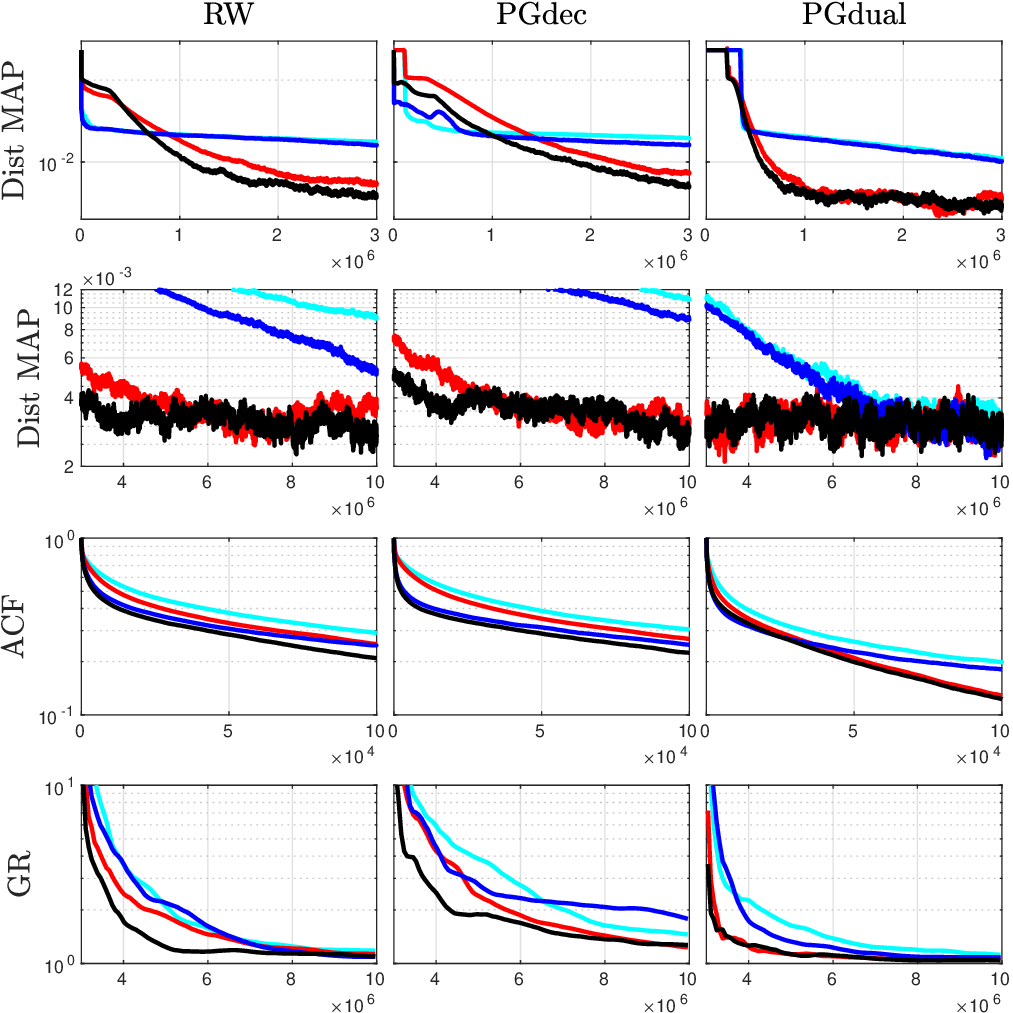}
\caption{\label{fig:asymptotique:UK} Displayed: the normalized
  distance to $\widehat{\Rvect}_{\mathrm{MAP}}$ versus the iteration
  index, during the burn-in phase (row 1) and after the burn-in phase
  (row 2); the ACF criterion versus the lag (row 3) and the GR
  criterion versus the iteration index (row 4); for the {\tt RW}-based
  samplers (column 1), {\tt PGdec}-based samplers (column 2) and {\tt
  PGdual}-based samplers (column 3). 
  For each {\tt Algo} in \{ {\tt RW}, {\tt PGdec}, {\tt PGdual} \}: 
    {\textcolor{cyan}{{\tt Algo I} is in cyan}},  
  {\textcolor{red}{{\tt Algo O} is in red}}, 
  {\textcolor{blue}{{\tt Gibbs Algo I} is in blue}}
  and {\tt Gibbs Algo O} is in black.}
\end{figure}

\subsection{Performance comparisons.}
\label{ssec:performances}
\noindent  {\bf Covariance matrices.} 
Normalized Distance to MAP indices (Fig.~\ref{fig:asymptotique:UK}[rows 1 \& 2]) and GR indices (Fig.~\ref{fig:asymptotique:UK}[row 4]) show 
that, for all algorithms {\tt RW}, {\tt PGdec}, and {\tt PGdual}, the strategy {\tt O} is more efficient than the strategy {\tt I}, as corresponding indices
decay more rapidly (to $0$ and $1$ respectively) for the formers than
for the latters. 
A plausible explanation is that the smaller condition number of $\barDo$ compared to $\barDinv$ is crucial to improve the mixing behavior of the Markov chains.\\
\noindent  {\bf Metropolis-Hastings vs. Gibbs samplers.} The evolution of the ACF criteria (Fig.~\ref{fig:asymptotique:UK}[row 3])
shows that the Gibbs strategies are globally more efficient than the
Metropolis-Hastings ones.  Further, GR indices confirm better
efficiency of the Markov chains obtained from {\tt Gibbs O}
strategies.\\
\noindent {\bf Drift functions.} The benefit of using functions $\stepalgo$ miming optimization
algorithms to drift the proposed points towards the higher probability
regions of $\pi_\Zvect$ is clearly quantified in Fig.~\ref{fig:asymptotique:UK} across all performance indices. 

During the burn-in phase, the normalized distance to MAP indices decay
toward $0$ significantly faster for all {\tt PGdual}-based algorithms,
compared to {\tt RW}-based ones.  Also, the {\tt PGdual O} algorithms
reach an expected plateau early in the burn-in phase, while this is
barely the case at the end of the burn-in phase for {\tt RW O}
algorithms.
 
After the burn-in phase (second row), the samplers using
optimization-based drift functions $\stepalgo$ show better behavior
after reaching the high density regions as they permit a broader and
more rapid exploration of the support of the distribution around its
maximum, with large amplitude moves from
$\widehat{\Rvect}_{\mathrm{MAP}}$, and faster returns to
$\widehat{\Rvect}_{\mathrm{MAP}}$.  This is notably clearly visible
for the {\tt Gibbs PGdual O} strategy.

ACF indices decay more rapidly for the {\tt PGdual} strategies than
for the {\tt RW} ones, irrespective of the choice of the covariance
matrix or of the Metropolis-Hastings or Gibbs versions.  Also, ACF
indices for {\tt PGdual} algorithms are less sensitive to the choice
of the covariance matrices than the {\tt RW} ones.

Finally, the GR statistic indices (row 4) clearly illustrates that
Markov chains produced by {\tt PGdual} samplers have better mixing
properties, compared to others, showing hence sensitivity to the
choice of the initial point. \\
\noindent  {\bf Optimal sampler.} Combined together, these observations, globally consistent with those stemming from other countries or time periods, lead to the following generic comments. 

While yielding essentially equivalent performance, across time periods
and countries, the choice of the covariance matrices (algorithms {\tt
O} and {\tt I}) significantly outperforms algorithms relying on
Identity covariance matrices (not reported here as showing poor
performances). This non trivial construction actually stemmed from the
thorough and detailed mathematical analysis conducted
in \autoref{sec:PGdecAndPGdual}. \Cref{theo:driftdual} advocates to
perform the augmentation of the $(T-2) \times T$ matrix $\D2$ into a
$T \times T$ invertible matrix $\barD2$, by adding two rows which are
orthogonal to the rows of $\D2$.

The observation that Gibbs samplers show better performances may stem from the fact that they
benefit from larger values of the step sizes learned on the fly by the
algorithms which favor larger jumps when proposing
$\paramvect^{n+1/2}$ from $\paramvect^n$ and imply better mixing
properties.  

Finally, the {\tt PGdual} algorithms show homogeneous performances for
example when varying the covariance matrix, and are thus less
sensitive to parameter tuning, an important practical feature.  Also,
the {\tt PGdual} algorithms show systematically better performances
than the {\tt PGdec} ones. This may result from the definition of
$\stepalgo^{\mathrm{PGdec}}$ which, because of the block-splitting
approach (see \eqref{eq:driftPGD:R:covid}), uses only partial
information on $\pi_\Zvect$ at each iteration.

As an overall conclusion, systematically observed across all studied
time periods and countries, the {\tt Gibbs PGdual O} algorithm,
devised from the careful analysis of the theoretical properties of the
distributions defined by
A\autoref{ass:smooth}-A\autoref{ass:nonsmooth}, is consistently found
to be the most efficient strategy for a relevant assessment of the
Covid19 pandemic time evolution.

\section{Credibility intervals for $\R$} 
\label{sec:results}
\noindent {\bf Goal and set-up.} The present section aims to illustrate the relevance of credibility interval-based estimations for the reproduction number $\R$ and for the outliers $\O$ from real Covid19 data. The statistical model for $(\Rvect, \Ovect)$ is described in \Cref{sec:statistic:model}.
Because it is of greater interest for epidemiologists, the
credibility intervals are translated into credibility intervals for
the denoised counts $\Zvect^{(D)}$ by their simple subtraction to the
original counts $\Zvect$, i.e., intuitively $\Zvect^{(D)}
= \Zvect-\Ovect$.

Credibility interval-based estimations are reported for {\tt Gibbs PGdual O} sampling scheme only, as 
Section~\ref{sec:assessment} established that it achieved the best performance amongst the twelve sampling schemes tested. 
Estimations are computed for a time period of five weeks ($T=35$ days), which corresponds to a few typical pandemic time scales, of the order of $7$ days, induced by the
serial interval function $\Phi$, cf. Section~\ref{sec:model}.
This period is set to a recent phase of the pandemic (Dec. 13th, 2021 to Jan. 17th, 2022). 
Estimations are reported for several countries, arbitrarily chosen as representatives of the pandemic across the world, 
but conclusions are valid for most countries.

\noindent {\bf Computation of the credibility intervals.} Credibility intervals are computed as follows: 
For each day $t \in [T_1, T_2]$ in the period of interest $[T_1,
T_2]$, the chosen sampling scheme ({\tt Gibbs PGdual O}) outputs
$7.10^6$ points of a Markov chain approximating the 
distribution of $\R_t$~; Quantiles of that distribution are estimated
using the empirical cumulative distribution function~; For a chosen
credibility level of $1-\alpha$, the upper and lower limits of the
credibility intervals are defined by the empirical $1-\alpha/2$ and
$\alpha/2$ quantiles.  For illustration purposes, $1-\alpha$ is set
here to $95\%$.  The same procedure is applied to produce credibility
intervals for the outliers $\O_t$.  Finally, credibility intervals for
the estimated denoised new infection counts $\Zvect^{(D)}$ are obtained by
subtracting the credibility intervals for $\O_t$ to the  count $\Z_t$.  Fig.~\ref{fig:CIsbb} reports, top
plots, the daily counts of new infections (black lines) to which are
superimposed the $95\%$ credibility interval-based estimations for the
denoised counts, $\Zvect^{(D)}$, (red pipes).  Fig.~\ref{fig:CIsbb} further
reports, bottom plots, the $95\%$ credibility interval-based
estimations for $\R$. 

\begin{figure}[t]
\centerline{
\includegraphics[width=0.5\linewidth]{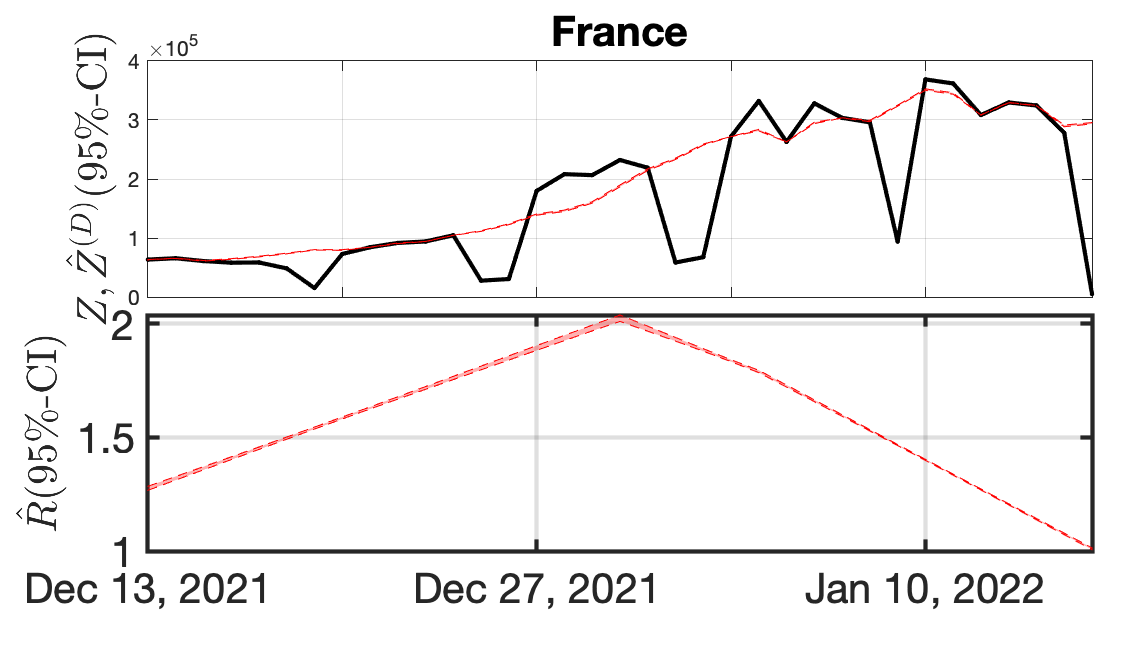}
\includegraphics[width=0.5\linewidth]{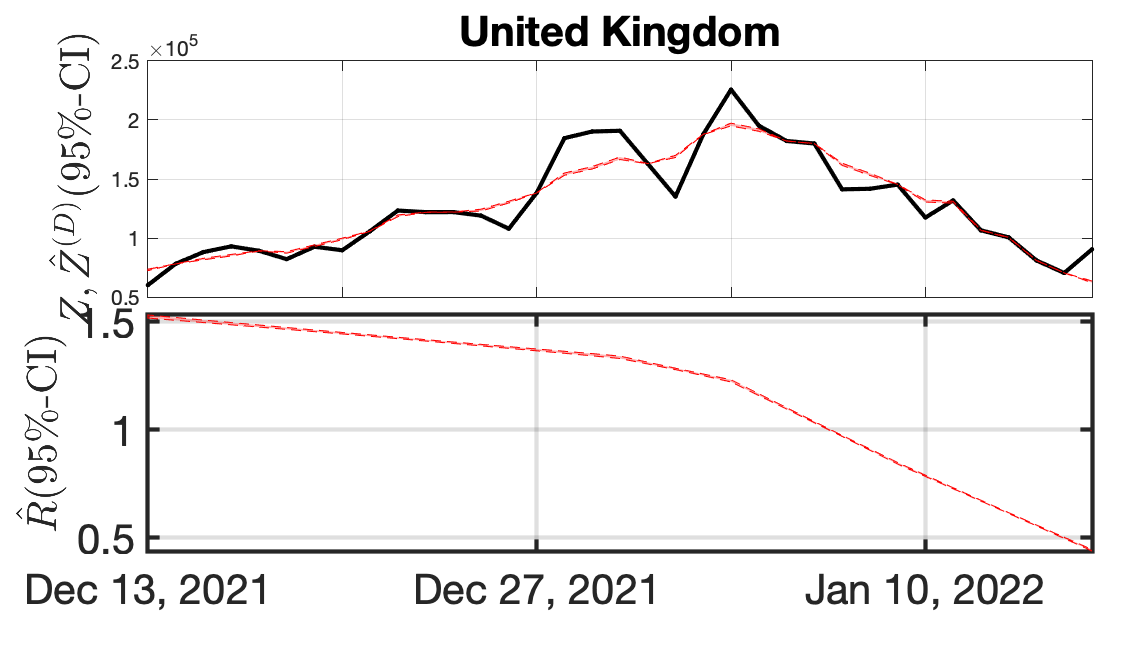}
}
\centerline{
\includegraphics[width=0.5\linewidth]{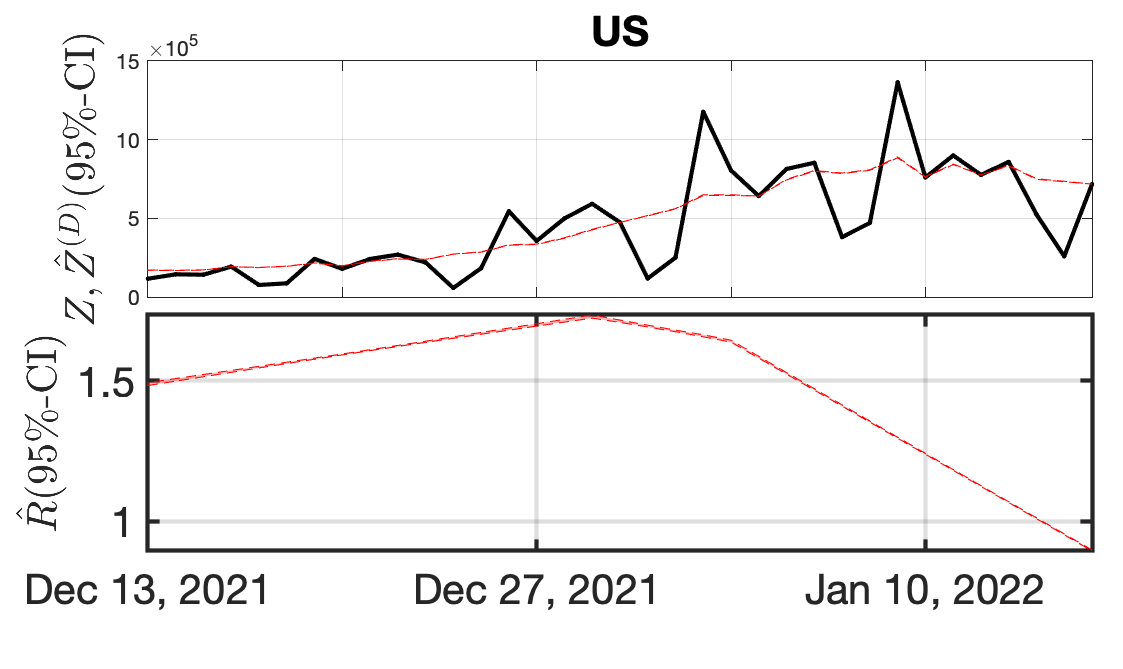}
\includegraphics[width=0.5\linewidth]{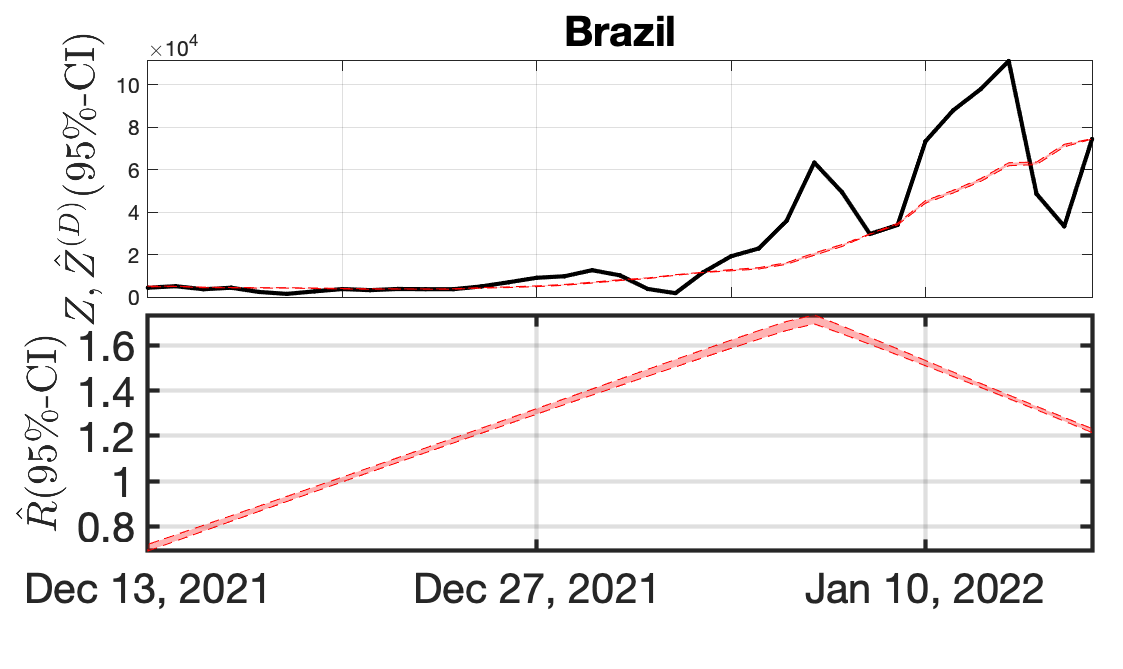}
}
\centerline{
\includegraphics[width=0.5\linewidth]{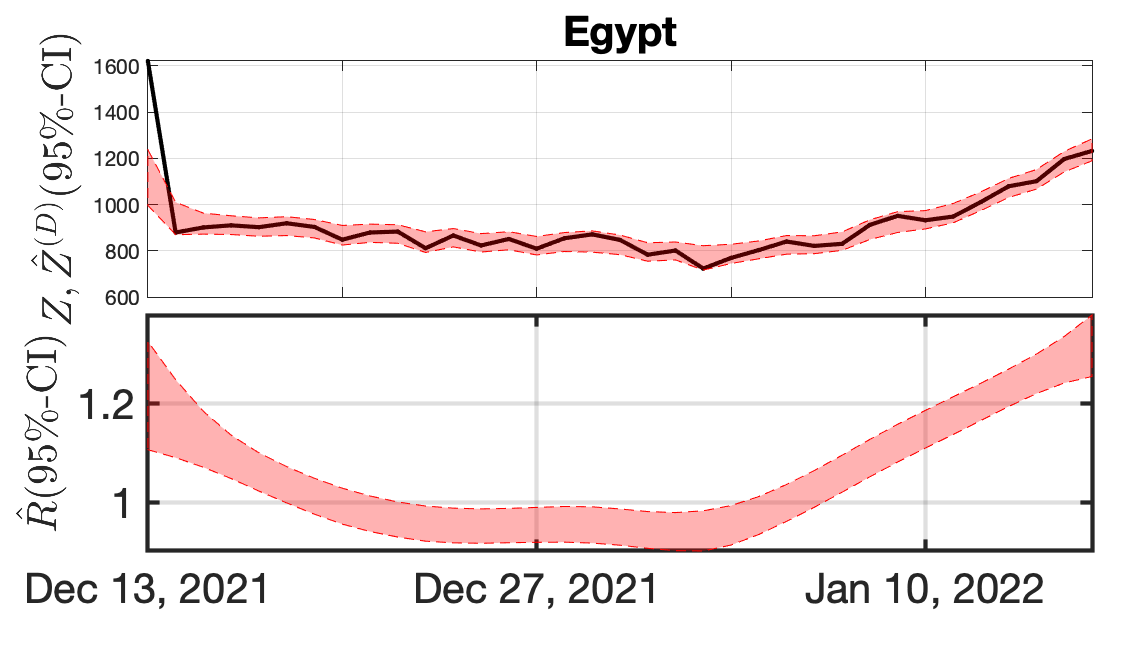}
\includegraphics[width=0.5\linewidth]{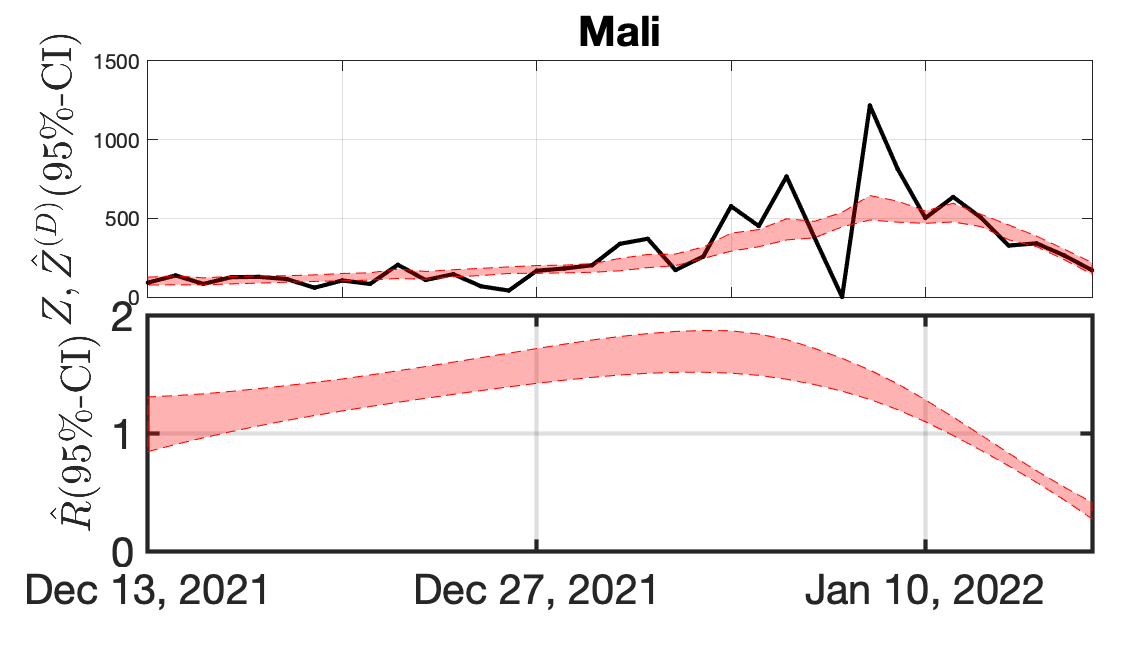}
}
\centerline{
\includegraphics[width=0.5\linewidth]{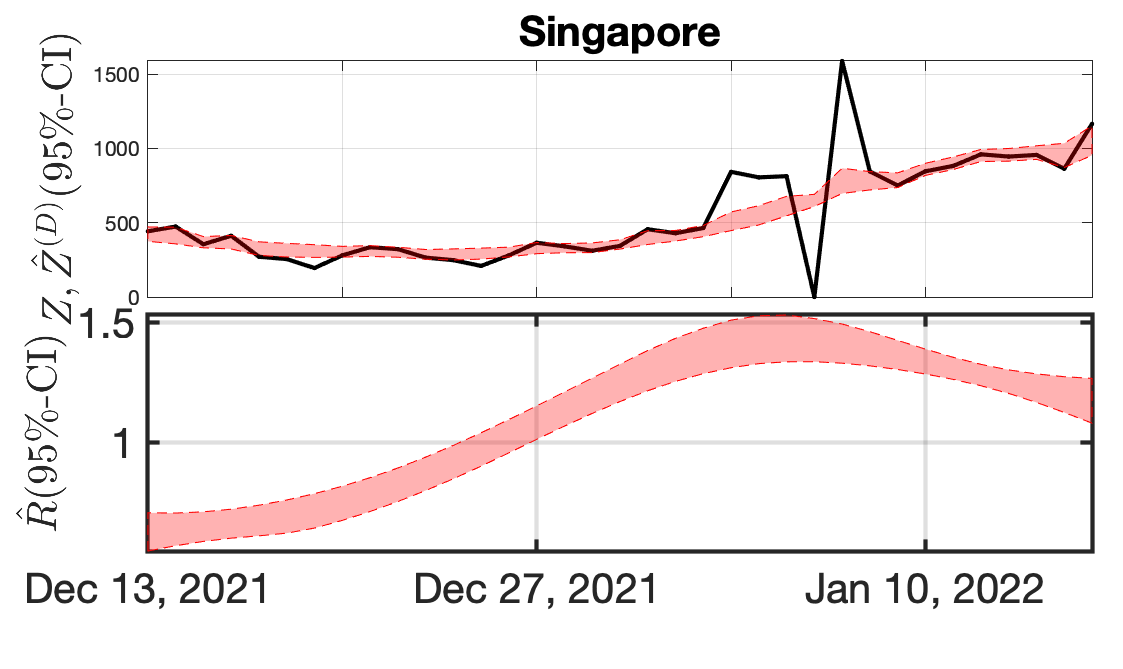}
\includegraphics[width=0.5\linewidth]{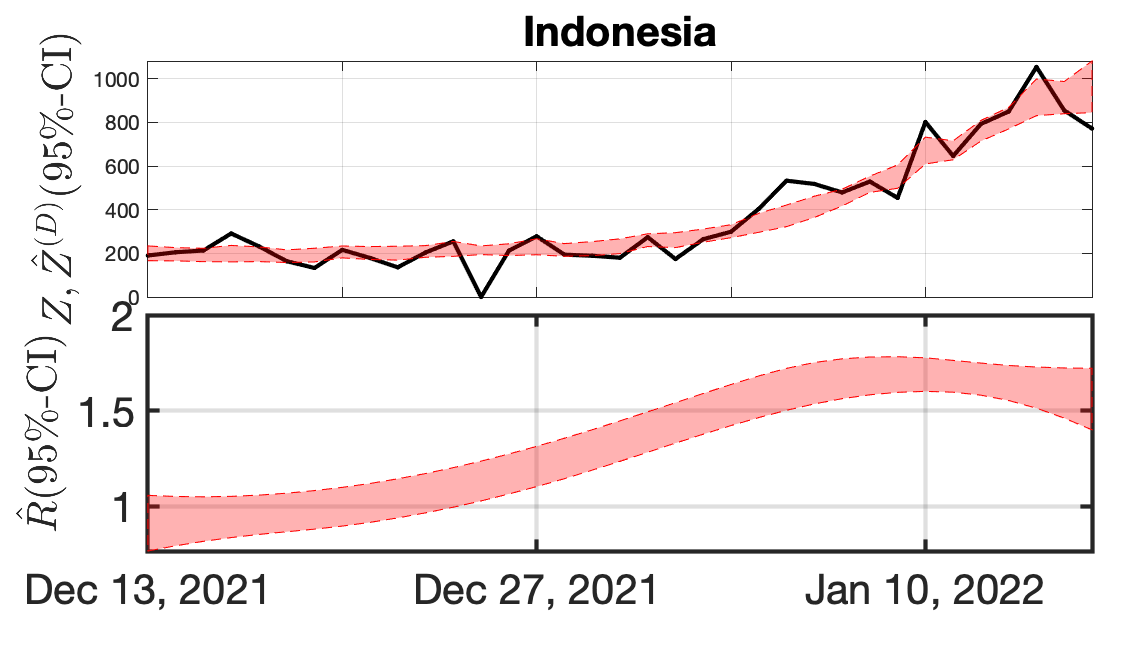}
}
\caption{\label{fig:CIsbb} {\bf {\tt Gibbs PGdual O} sampler-based estimation of the time evolution of $\R$ for a recent 5-week time period and for several countries.} 
Top rows: Raw daily new infections counts $\Zvect$ (black) and
estimates of the denoised counts $\Zvect^{(D)}$, obtained by
subtracting the $95\%$-Credibility interval estimates of the outliers
to the raw new infection counts $\Zvect$~; Bottom rows:
$95\%$-Credibility interval estimated of $\R_t$.
}
\end{figure}

\noindent {\bf Relevance of credibility interval-based estimations.} Fig.~\ref{fig:CIsbb}, together with the examination of equivalent plots for other countries, yields the following generic conclusions.

The estimated denoised counts show far smoother evolution along time
compared to the raw counts, hence providing far more realistic
assessments of the intensity of the pandemics.  Notably, for most
countries, the zero or low counts, associated with week-ends or
non-working days, followed by days with over-evaluated counts by
compensation, are smoothed out by the outlier estimation procedure,
while the values of the counts for the regular (or non corrupted) days
are left unchanged.  This is the direct benefit of the nonlinear
filtering procedure underlying the estimation formulation
in \eqref{eq:target:covid}, as opposed to traditional denoising
procedures performed by classical linear filtering (such as moving
average).

For the credibility intervals for $\R$, their sizes range, depending on
countries and time periods, from below $1 \%$ to above $10 \%$.  A
careful examination shows that the credibility interval size is mostly
driven by data quality: the credibility interval size increases when
outliers are detected.  Further, for a given country, the size of the
credibility intervals varies only mildly along time over a five-week
period.  Changes in size are often associated with changes in the
trends of the estimates of $\R$, or with the occurrence of outliers.
These credibility intervals provide hence a relevant assessment not
only of the intensity of the pandemics, but also of the confidence
that can be granted to this assessment, by providing epidemiologists
with a range of likely values of $\R$, rather than a single value.
This permits to compare the evolution of the pandemics across several
countries on a better scientifically grounded basis.

These credibility interval-based estimations permit a double analysis
of the pandemic: They permit retrospectively to evaluate the impacts
of sanitary measures on the pandemic evolution. Additionally, the
smooth nature of the estimation of $\R$ (close to piecewise
linear) performs an implicit short term forecast (or a \emph{nowcast})
of the evolution of the pandemic intensity: For instance, for several
countries (e.g., France, Mali, Brazil, Singapore,\ldots) the estimate
of $\R$ is decreasing for the last $5$ to $10$ days of the studied
period, predicting that daily new infections will reach a
maximum of the current wave within the coming days and then will start
to decrease.

These credibility interval estimates can be complemented with other estimates such as the Maximum, median or Mean a Posteriori (cf. e.g., \cite{artigas2021credibility,abry:etal:EMBC}). 

Finally, let us emphasize that these estimates, denoised counts and credibility intervals, are obtained using a single and same set of hyperparameters $\lambdatime$ and $\lambdaO$ common to all countries.

\section{Conclusions}
\label{sec:conclusions}
 The proposed tools perform a relevant credibility interval-based
estimation of the Covid19 reproduction number and denoised new
infection counts, by combining a statistical modeling of the time
evolution of the pandemic with Monte Carlo Metropolis 
sampling strategies.  Robustness against the low quality of the Covid19 is
achieved by engineering the Bayesian model to impose sparsity in the
changes of a smooth time evolution for the reproduction number and in
the outlier occurrences, modeling data corruption.  This is obtained
at the price of the non-differentiability of their a posteriori
distribution, thus precluding the use of the classical Metropolis
Adjusted Langevin Algorithm to produce credibility intervals.  This
lead us to propose several Proximal-Gradient Monte Carlo algorithms
tailored to the sampling of non-differentiable 
distributions, that also constitute valid sampling schemes for a much
broader range of applications than that of the strict Covid19 pandemic monitoring, e.g., in image processing or more
generally, in any Bayesian inverse problems with several nonsmooth
priors.

Estimation performances were assessed and compared on real Covid19
data, using a set of well-selected indices quantifying the efficiency
of these different sampling schemes.

Finally, it was shown for several countries and for a recent five-week
 period that the achieved credibility interval-based estimations
of both denoised new infection counts and reproduction number provide
practitioners with an efficient and robust tool for the actual and
practical monitoring of the Covid19 pandemic.  Such estimates are
updated on a daily basis on the authors's web-pages.  Automated and
data-driven estimations of the hyperparameters $\lambdatime$ and
$\lambdaO$ are under current investigations.

In an effort toward reproducible research and open science, MATLAB
codes implementing {\tt PGdec}, {\tt PGdual} for both
Metropolis-Hastings and Gibbs versions are made publicly available at
\url{https://github.com/gfort-lab/OpSiMorE}.

%
%
%
%
%

\clearpage
\newpage

\begin{center}
 \Large  Supplementary material
\end{center}

\section{Proof of \cref{prop:MAPexists}}
\label{sec:SM:proofMAP}
\subsection{Notations}
\def\tildeRvect{\widetilde{\Rvect}} Let $\Dset_\Zvect$ be the subset
of $(\rset_+)^T \times \rset^T$ given by \eqref{eq:const}. To shorten
the notations, we will write $\mathcal{I}_t(\paramvect)$ instead of $\mathcal{I}_t(\paramvect,\Zvect)$:
\begin{equation}\label{eq:intensity:app}
\mathcal{I}_t(\paramvect) \eqdef \R_t \Zphi_t + \O_t, \qquad \Zphi_t \eqdef
\sum_{u=1}^{\tau_\phi} \Phi_u \Z_{t-u};
\end{equation}
and $\pi$ instead of $\pi_\Zvect$ (see
\eqref{eq:target:covid}) . Observe that $- \ln \pi(\paramvect)$ is
equal to $+\infty$ for $\paramvect \notin \Dset_{\Zvect}$ and for
$\param \in \Dset_{\Zvect}$
\begin{align*}
-\ln \pi(\paramvect) & = C_\pi + \sum_{t=1}^T\left\{
\mathcal{I}_t(\paramvect) - \Z_t \1_{\Z_t >0 } \ln \mathcal{I}_t(\paramvect) \right\}
\\ &+ \lambdatime \| \D2 \Rvect \|_1 + \lambdaO \| \Ovect\|_1;
\end{align*}
for some normalizing constant $C_\pi$. By convention $0 \ln 0 = 0$.
Define the $T \times T$ invertible matrix $\barD2$ and compute its
inverse:
\begin{align*}
& \barD2 \eqdef \left[\begin{matrix}
  1 & 0 & 0  & \cdots & 0 \\
  -2 & 1 & 0 & \cdots & 0 \\
  & & \D2 & & 
  \end{matrix} \right]; \\
& \barD2^{-1} \eqdef \left[\begin{matrix} 1 & 0 & 0 & \cdots & 0 \\ 2 &
    1 & 0 & \cdots & 0 \\ \cdots \\ T & (T-1) & \cdots & 2 & 1
  \end{matrix} \right].
\end{align*}
Finally, define the criterion
\begin{multline} \label{criterion}
  \mathcal{C}(\tildeRvect, \Ovect) \eqdef \lambdatime \|
  \tildeRvect_{3:T} \|_1 + \lambdaO \| \Ovect\|_1 \\ +
  \sum_{t=1}^T\left\{ \mathcal{I}_t(\barD2^{-1} \tildeRvect, \Ovect) - \Z_t
  \1_{\Z_t>0} \ln \mathcal{I}_t(\barD2^{-1} \tildeRvect, \Ovect) \right\},
\end{multline}
for $(\tildeRvect, \Ovect)$ in the set $\widetilde \Dset_{\Zvect}
\eqdef \{ (\tildeRvect, \Ovect) \ \text{s.t.} \ (\barD2^{-1}
\tildeRvect, \Ovect) \in \Dset_{\Zvect} \}$ and $+\infty$ otherwise.
We have for any $(\tildeRvect, \Ovect) \in \widetilde \Dset_{\Zvect}$
\begin{equation} \label{eq:fromCtopi}
\mathcal{C}(\tildeRvect, \Ovect) = -\ln \pi(\barD2^{-1} \tildeRvect,
\Ovect) - C_\pi.
\end{equation}

\subsection{Existence of a MAP}
We start with two lower bounds on the criterion $\mathcal{C}$ which
will help us to study its behavior on some boundaries of $\widetilde
\Dset_\Zvect$.
\noindent {\bf $\blacktriangleright$ Lower bounds on $\mathcal{C}$.}  Let $(\tildeRvect, \Ovect) \in
\widetilde \Dset_{\Zvect}$. Then
  \begin{equation}\label{eq:minoration:C:1}
\mathcal{C}(\tildeRvect, \Ovect) \geq - \sum_{t=1}^T \ln (\Z_t!) +
\lambdatime \|(\tildeRvect)_{3:T}\|_1 + \lambdaO \| \Ovect \|_1,
\end{equation}
and for any $ \tau \in \{1, \cdots, T \}$,
\begin{multline}\label{eq:minoration:C:2}
\mathcal{C}(\tildeRvect, \Ovect) \geq \left\{ \mathcal{I}_\tau(\barD2^{-1}
\tildeRvect, \Ovect) - \Z_\tau \1_{\Z_\tau >0 } \ln \mathcal{I}_\tau(\barD2^{-1}
\tildeRvect, \Ovect) \right\} \\ - \sum_{t \neq \tau} \ln (\Z_t!).
\end{multline}
\begin{proof}
For any $p >0$ and $z \in \nset$, $p^z \exp(-p) / z! \in \ooint{0,1}$
thus implying that $z \ln p - p -\ln(z!) \leq 0$.  When $p=z=0$, then
$z \1_{z>0} \ln p - p = 0$ and $\ln(z!) = 0$; hence we have
\[
\mathcal{I}_t(\barD2^{-1} \tildeRvect, \Ovect) - \Z_t \1_{\Z_t >0 } \ln
\mathcal{I}_t(\barD2^{-1} \tildeRvect, \Ovect) \geq - \ln (\Z_t!),
\]
from which we obtain \eqref{eq:minoration:C:1} and
\eqref{eq:minoration:C:2}.
  \end{proof}

\noindent {\bf $\blacktriangleright$ Behavior of $\mathcal{C}$ on some
  boundaries of $\widetilde \Dset_\Zvect$.}  Assume that there exist
$t_\star < t_{\star \star}$ in $\{1, \ldots, T \}$ such that
$\Zphi_{t_\star} > 0$ and $\Zphi_{t_{\star\star}} > 0$.
  \begin{enumerate}
\item For any sequence $(\tildeRvect^n, \Ovect^n) \in \widetilde
  \Dset_{\Zvect}$ s.t. $\lim_n \| \tildeRvect^n \|_1 + \lim_n
  \|\Ovect^n\|_1 = +\infty$, we have $\lim_n
  \mathcal{C}(\tildeRvect^n, \Ovect^n) = + \infty$.
\item Let $t \in \{1, \ldots, T \}$ such that $\Z_t>0$.  For any
  sequence $(\tildeRvect^n, \Ovect^n) \in \widetilde \Dset_{\Zvect}$
  s.t. $\lim_n (\barD2^{-1} \tildeRvect^n)_t \Zphi_t + \Ovect^n_t =
  0$, we have $\lim_n \mathcal{C}(\tildeRvect^n, \Ovect^n) = +
  \infty$.
    \end{enumerate} 
\begin{proof} {\it First statement.}
Let $\{(\tildeRvect^n, \Ovect^n), n \geq 0\}$ be a sequence in
$\widetilde \Dset_{\Zvect}$. For the discussions below, remember that
this implies that
\begin{equation}\label{eq:limit:inD}
\lim_n \left( (\barD2^{-1} \tildeRvect^n)_t \Zphi_t + \Ovect_t^n
\right) \geq 0.
\end{equation}
We distinguish two cases. First, either $\|\Ovect^n\|_1$ tends to
infinity or $\|\Rvect^n_{3:T}\|_1$ tends to infinity. In the second
case, these two norms are assumed bounded but $\|\Rvect^n_{1:2}\|_1$
tends to infinity. \\ {\tt $\bullet$ First case.}  Assume first
$\lim_n \{ \|\tildeRvect^n_{3:T}\|_1 + \| \Ovect^n \|_1 \}=
+\infty$. Then, from \eqref{eq:minoration:C:1}, we have $\lim_n
\mathcal{C}(\tildeRvect^n, \Ovect^n) = +\infty$. \\
        \noindent {\tt $\bullet$ Second case.} Now consider the case
        when
  \[
\sup_n \left( \|\tildeRvect^n_{3:T}\|_1 \sup_n \| \Ovect^n \|_1
\right)< \infty, \quad \lim_n \|\tildeRvect^n_{1:2} \|_1 = +\infty.
\]
By definition of $\barD2^{-1}$ , we have for any $t \in \{1, \ldots, T
\}$,
\begin{equation}\label{eq:case2A}
(\barD2^{-1} \tildeRvect^n)_{t} = t \tildeRvect_1^n + (t-1)
\tildeRvect_2^n + \sum_{k=3}^{t} (t-k+1) \tildeRvect_k^n;
\end{equation}
by convention, the last term in the RHS is zero when $t=1,2$.  Under
the assumptions of this second case, $\sup_n | \sum_{k=3}^{t} (t-k+1)
\tildeRvect_k^n| < \infty$ for any $t \in \{3, \ldots, T \}$. We prove
that either $(\barD2^{-1} \tildeRvect^n)_{t_\star}$ tends to infinity,
or $(\barD2^{-1} \tildeRvect^n)_{t_{\star \star}}$ tends to infinity
-- which will imply by \eqref{eq:minoration:C:2} that the criterion
tends to infinity. \\
\noindent {\tt $\bullet$ Subcase 2A.}  Assume that $\lim_n \{t_\star
\tildeRvect_1^n + (t_\star-1) \tildeRvect_2^n \}= +\infty$ (observe
that this limit can not be $-\infty$, as a consequence of the
assumptions of Case 2, \eqref{eq:case2A} and
\eqref{eq:limit:inD}). \\ Apply \eqref{eq:case2A} with $t = t_\star$;
this yields $\lim_n (\barD2^{-1} \tildeRvect^n)_{t_\star} =
+\infty$. Since $\Zphi_{t_\star>0}$ then $\lim_n
\mathcal{I}_{t_\star}(\barD2^{-1}\tildeRvect^n, \Ovect^n) = + \infty$ (see
\eqref{eq:intensity:app}) which implies that
\[
\lim_n \left\{ \mathcal{I}_{t_\star}(\barD2^{-1}\tildeRvect^n, \Ovect^n) - \Z_{t_\star} \ln
\mathcal{I}_{t_\star}(\barD2^{-1}\tildeRvect^n, \Ovect^n) \right\} = +\infty
\]
and then $\lim_n \mathcal{C}(\tildeRvect^n, \Ovect^n) = +\infty$ by
\eqref{eq:minoration:C:2}. \\ {\tt $\bullet$ Subcase 2B.} Assume that
$\sup_n |t_\star \tildeRvect_1^n + (t_\star-1) \tildeRvect_2^n|
<\infty$.  Then, necessarily $\lim_n |\tildeRvect_2^n| = +\infty$
(otherwise, it is the Subcase 2A).  We write
\begin{align*}
  (\barD2^{-1} \tildeRvect^n)_{t_{\star \star}}&- \sum_{k=3}^{t_{\star
      \star}} (t_{\star \star}-k+1) \tildeRvect_k^n \\ & = t_{\star
    \star} \tildeRvect_1^n + (t_{\star \star}-1) \tildeRvect_2^n 
  \\ & = \frac{t_{\star \star}}{t_\star}\left(
  t_\star \tildeRvect_1^n + (t_\star-1) \tildeRvect_2^n \right) +
  \frac{t_{\star \star}}{t_\star}(1 - \frac{t_\star}{t_{\star \star}}
  ) \tildeRvect_2^n.
\end{align*}
Since $t_\star < t_{\star \star}$, this equality and
\eqref{eq:limit:inD} imply that $\lim_n \tildeRvect_2^n = +\infty$.
Therefore, since $t_\star < t_{\star \star}$, we have $\lim_n
(\barD2^{-1} \tildeRvect^n)_{t_{\star \star}} = + \infty$. We then
conclude, along the same lines as in Subcase 2A, that $\lim_n
\mathcal{C}(\tildeRvect^n, \Ovect^n) = +\infty$. \\
        \noindent {\it Second statement.}  Let $\{(\tildeRvect^n, \Ovect^n), n \geq 0\}$
be a sequence in $\widetilde \Dset_{\Zvect}$ and $\tau$ such that
$\Z_\tau >0$.  By \eqref{eq:minoration:C:2}, we have
\begin{multline*}
\mathcal{C}(\tildeRvect^n, \Ovect^n) + \sum_{t \neq \tau} \ln (\Z_t!)
\\ \geq \mathcal{I}_\tau(\barD2^{-1} \tildeRvect^n, \Ovect^n) - \Z_\tau \ln
\mathcal{I}_\tau(\barD2^{-1} \tildeRvect^n, \Ovect^n).
\end{multline*}
The RHS tends to $+\infty$ since $\lim_n \mathcal{I}_\tau(\barD2^{-1}
\tildeRvect^n, \Ovect^n) = 0$ and $\Z_\tau>0$ by assumptions, hence,
$\lim_n \mathcal{C}(\tildeRvect^n, \Ovect^n) = + \infty$.
\end{proof}
\noindent {\bf $\blacktriangleright$ Conclusion: a MAP exists.}
Assume that there exist $t_\star < t_{\star \star}$ in $\{1, \ldots, T
\}$ such that $\Zphi_{t_\star} > 0$ and $\Zphi_{t_{\star\star}} >
0$. Then $- \ln \pi$ possesses at least one minimizer in
$\Dset_{\Zvect}$.
\begin{proof}
Consider the point $(\Rvect^0, \Ovect^0)$ given by $\Rvect^0 \eqdef
(1, \ldots, 1)^\top$ and $\Ovect^0 \eqdef (1,\ldots, 1)^\top$. Then,
$(\Rvect^0, \Ovect^0) \in \Dset_{\Zvect}$. Set $M^0 \eqdef - \ln
\pi(\Rvect^0, \Ovect^0)$. \\ The goal of the proof below is to build a
closed bounded set $\mathcal{K}$ in $\Dset_{\Zvect}$ such that outside
$\mathcal{K}$, $-\ln \pi \geq 1+M^0$. This implies that $(\Rvect^0,
\Ovect^0) \in \mathcal{K}$. Since $- \ln \pi$ is continuous on
$\Dset_{\Zvect}$, it reaches its minimum on the compact subset
$\mathcal{K}$, and this minimum is upper bounded by $M^0$. Hence, this
minimizer is also a global minimizer.  Let us define $\mathcal{K}$. We
have
 \[
\lambda_{\min} \|\tildeRvect \|^2 \leq \| \D2^{-1} \tildeRvect \|^2 =
\|\Rvect\|^2 \leq \lambda_{\max} \|\tildeRvect \|^2
  \]
where $\lambda_{\min}$ (resp. $\lambda_{\max}$) is the minimal
(resp. maximal) eigenvalue of $\barD2^{-\top} \barD2^{-1}$; they are
positive and finite. Consequently, by setting $\tildeRvect^n \eqdef
\barD2 \Rvect^n$, we have $\|(\tildeRvect^n, \Ovect^n)\|_\ell \to +
\infty$ iff $\|(\Rvect^n, \Ovect^n)\|_\ell \to + \infty$ for $\ell =
1,2$ since the norms are equivalent on $\rset^{2T}$. This property,
the coercivity property (statement 1), and the equality
\eqref{eq:fromCtopi} imply that $\lim_n - \ln \pi(\Rvect^n, \Ovect^n)
= + \infty$ for any $\Dset_{\Zvect}$-valued sequence $\{(\Rvect^n,
\Ovect^n), n \geq 0\}$ such that $\lim_n \| \Rvect^n \| + \lim_n \|
\Ovect\|^n = +\infty$. As a consequence, there exists $C_{1+M^0}$ such
that
\begin{multline*}
(\Rvect, \Ovect) \in \Dset_{\Zvect}, \| \Rvect\| + \| \Ovect\| >
  C_{1+M_0} \\ \Longrightarrow - \ln \pi(\Rvect, \Ovect) \geq 1+M^0.
\end{multline*}
Similarly, there exists $c_{1+M_0}>0$ such that
\begin{multline*}
(\Rvect, \Ovect) \in \Dset_{\Zvect}, \Rvect_t \Zphi_t+ \Ovect_t <
  c_{1+M^0} \ \text{for some $t$ s.t. $\Z_t >0$} \\ \\ \Longrightarrow
  - \ln \pi(\Rvect, \Ovect ) \geq 1+M^0.
\end{multline*}
Consequently, we define
\begin{multline*}
\mathcal{K} \eqdef \Dset_{\Zvect} \cap \{\paramvect: \| \Rvect\| + \|
\Ovect\| \leq C_{1+M_0} \} \\ \cap \{\paramvect: \Rvect_t \Zphi_t+
\Ovect_t \geq c_{1+M^0} \ \text{for $t$ s.t. $\Z_t >0$} \}.
\end{multline*}
\end{proof}
\vspace{-0.9cm}
\subsection{About the uniqueness of the MAP}
\def\sign{\mathrm{sign}} We just proved that there exists a compact
subset of the interior of $\Dset_\Zvect$ that contains a minimizer of
$-\ln \pi$. Let $\paramvect^\star = (\Rvect^\star, \Ovect^\star)$ be a
minimizer. \\
\noindent {\bf $\bullet$ One or uncountably many.} The function $-\ln
\pi$ is convex and finite on a convex set: hence, given a second
minimizer $\paramvect^{\star \star}$, $ \mu \paramvect^\star +(1-\mu)
\paramvect^{\star \star}$ is also a minimizer, whatever $\mu \in
\ccint{0,1}$. \\ \noindent {\bf $\bullet$ Same intensity, data
  fidelity term, and penalty term. } Let $f_\Zvect$ and $g$ be defined
by \eqref{def:fandg:g}, \eqref{def:fandg:f}. Following the same lines as in \cite{Pascal2021}
where the strict convexity of the Kullback-Leibler term $f_\Zvect$ is
the key ingredient, it can be proved that $\mathcal{I}_t(\paramvect^\star) =
\mathcal{I}_t(\paramvect^{\star \star})$ for any $t \in \{1, \ldots, T\}$ and
thus $f_\Zvect(\paramvect^\star) = f_\Zvect(\paramvect^{\star
  \star})$; since $-\ln \pi(\paramvect^\star) = -\ln
\pi(\paramvect^{\star \star})$ since both points minimize $-\ln \pi$,
we have $g(\paramvect^\star) = g(\paramvect^{\star \star})$. \\
\noindent {\bf $\bullet$ Sign conditions.}
Set $\sign(a)= 1$ when $a>0$, $\sign(a) =-1$ when $a<0$ and $\sign(a)=
0$ when $a=0$. For $A,B$ in $\rset$, and $\mu >0$, we have $|A + \mu
B| = |A| + \mu \sign(A) \, \sign(B) |B| \1_{A \neq 0} + \mu |B| \1_{A=
  0} + \mu o(1)$ where $o(1)$ is a function satisfying $\lim_{\mu \to
  0} o(1)= 0$. Hence, for ${\boldsymbol \tau} = {\boldsymbol \tau}_{1:d},
{\boldsymbol \tau'}={\boldsymbol \tau}'_{1:d} \in \rset^d$ and $\mu \in
\ooint{0,1}$,
\begin{align} \label{eq:DL1:norm1}
\| {\boldsymbol \tau} &+ \mu ({\boldsymbol \tau}' - {\boldsymbol
  \tau})\|_1 = (1-\mu) \, \|{\boldsymbol \tau} \|_1 + \mu \|
\boldsymbol \tau' \|_1 \nonumber \\ &+ \mu \sum_{t=1}^d \left( \sign(\tau_t)
\sign(\tau_t') -1 \right) |\tau_t'| \1_{\tau_t \neq 0} + \mu \, o(1).
\end{align}
Set $\paramvect^\mu \eqdef \paramvect^\star + \mu(\paramvect^{\star
  \star} - \paramvect^\star)$; as proved above,
$f_\Zvect(\paramvect^\mu) = f_\Zvect(\paramvect^\star)$. We prove that
if the sign conditions do not hold, for $\mu$ small enough
$g(\paramvect^\mu) < g(\paramvect^\star)$ which yields a contradiction
since $\paramvect^\star$ is a minimizer.  By \eqref{eq:DL1:norm1}, we
obtain
\begin{align*}
 & \| \Ovect^\star \!+\! \mu ( \Ovect^{\star \star} \!-\! 
  \Ovect^\star) \|_1 \!-\! (1\!-\!\mu) \| \Ovect^\star \|_1 \!-\!
  \mu \|\Ovect^{\star \star} \|_1 \\ & = \mu \sum_{t=1}^T \left(
  \sign(\O^\star_t) \sign(\O_t^{\star \star}) -1 \right) |\O_t^{\star
    \star}| \1_{\O_t^{\star} \neq 0} + \mu \, o(1).
\end{align*}
We have a similar expansion for $\|\D2 \Rvect^\star + \mu(\D2
\Rvect^{\star \star} - \D2 \Rvect^\star) \|_1$. Since
$g(\paramvect^\star) = g(\paramvect^{\star \star})$, this yields
\begin{align*}
& g(\paramvect^\mu) - g(\paramvect^\star) = \mu \sum_{t=1}^{T-2} |(\D2
  \Rvect^{\star \star})_t| \ldots \\ & \times \left( \sign((\D2
  \Rvect^\star)_t) \ \sign( (\D2 \Rvect^{\star \star})_t) -1 \right)
  \1_{(\D2 \Rvect^{\star})t \neq 0} \\ & + \mu \sum_{t=1}^T \left(
  \sign(\O^\star_t) \ \sign(\O_t^{\star \star}) -1 \right)
  |\O_t^{\star \star}| \1_{\O_t^{\star} \neq 0} + \mu \, o(1).
\end{align*}
For $\mu$ small enough, the RHS is negative when for some $t$, $
\sign(\O^\star_t) \sign(\O_t^{\star \star}) = -1$ or when for some
$s$, $\sign((\D2 \Rvect^\star)_s) \sign( (\D2 \Rvect^{\star \star})_s)
= -1$. If such, $g(\paramvect^\mu) <
g(\paramvect^\star)$. \\ \noindent {\bf $\bullet$ Sufficient
  conditions for uniqueness.}  The proof is adapted from \cite[Section
  4]{ali:tibshirani:2019}. Let $\paramvect^{\star \star} =
\paramvect^\star + {\boldsymbol \omega}$ be another minimizer. Define
\[
\U\eqdef \left[\begin{matrix} \lambdatime \D2 & \0mat_{(T-2)
      \times T } \\ \0mat_{T \times T} & \lambdaO \Id_T \end{matrix}
  \right] \in \rset^{(2T-2) \times (2T)}.
    \]
The Fermat rule (\cite[Theorem 16.2]{Bauschke:2011ta}) which
characterizes optimality implies that zero is in the subdifferential
of $-\ln \pi$ at $\paramvect^\star$: there exists
$\modif{\boldsymbol{\gamma}}(\paramvect^\star) \in \rset^{2T-2}$ such that $\nabla
f_\Zvect(\paramvect^\star) + \U^\top \modif{\boldsymbol{\gamma}}(\paramvect^\star) = \0mat_{2T
  \times 1}$ where $\modif{\boldsymbol{\gamma}}(\paramvect^\star)$ is the subgradient of
the $L^1$-norm in $\rset^{2T-2}$ evaluated at $\U \paramvect^\star \in
\rset^{2T-2}$. Since $\nabla f_\Zvect(\paramvect^\star)$ depends on
$\paramvect^\star$ through the $\mathcal{I}_t$'s which are constant on the set
of the minimizers $\mathcal{M}$ (see above), and $\U^\top$ has full
column rank, $\modif{\boldsymbol{\gamma}}(\paramvect^\star)$ is the same whatever the
minimizer $\paramvect^\star$; it is denoted by $\modif{\boldsymbol{\gamma}}^\star$.  Set
$\mathcal{I} \eqdef \{j \in \{1, \ldots, 2T-2 \}: |\gamma^\star_j| < 1
\}$. Observe that any minimizer is in the kernel $\modif{\boldsymbol{\mathsf{K}}}_1$ of the
matrix $\U_{\mathcal{I}}$ which, by definition, collects the
rows of $\U$ indexed by $\mathcal{I}$; hence ${\boldsymbol \omega} \in
\modif{\boldsymbol{\mathsf{K}}}_1$. In addition, ${\boldsymbol \omega}$ is in the kernel
$\modif{\boldsymbol{\mathsf{K}}}_2$ of the $T \times (2T)$ matrix $[\mathsf{diag}(\Zphi_1,
  \ldots \Zphi_T) \, \Id_T]$, since all the $\mathcal{I}_t$'s are constant on
$\mathcal{M}$. Therefore, if $\modif{\boldsymbol{\mathsf{K}}}_1 \cap \modif{\boldsymbol{\mathsf{K}}}_2 = \{\0mat
\}$, the MAP is unique.

\section{Proof\modif{s} of \autoref{sec:sampler}}
\label{proof:PGD}
\subsection{Detailed proof of \autoref{theo:driftdual}}
Throughout the proof, we write $\A$, $\U$ and $\barA$ as a shorthand
notation for $\A_{i,j}$, $\U_{i,j}$ and $\barA_{i,j}$.  Under the
stated assumptions,
\begin{equation} \label{eq:def:invbarA}
\barA^{-1} = \begin{bmatrix} \U^\top (\U \U^\top)^{-1} & \A^\top
  (\A \A^\top)^{-1} \end{bmatrix}.
\end{equation}
We first focus on the gradient step in \eqref{eq:driftPGdual0} leading
to:
\[
\barA \paramvect_j - \pas_j \barA^{-\top} \nabla_j f(\paramvect) = \begin{bmatrix} \U \paramvect_j - \pas_j (\U \U^\top)^{-1} \U \nabla_j f(\paramvect) \\
  \A \paramvect_j - \pas_j (\A \A^\top)^{-1} \A\nabla_j f(\paramvect) \end{bmatrix}.
\]
Second, for any $\boldsymbol{\tau} = \boldsymbol{\tau}_{1:d_j} \in \rset^{d_j}$,
\begin{align}
\prox_{\pas_j \bar{g}_{i,j}}(\boldsymbol{\tau}) &= \prox_{\pas_j g_{i,j}(\A \barA^{-1} \cdot)}(\boldsymbol{\tau}) \\
&= \begin{bmatrix}
  \boldsymbol{\tau}_{1:d_j-c_{i,j}} \\ \prox_{\pas_j
    g_{i,j}}({\boldsymbol{\tau}}_{d_j-c_{i,j}+1:d_j}) \end{bmatrix},
\end{align}
 since under the stated assumptions, we have
\[
\A \barA^{-1} = \begin{bmatrix} \0mat_{c_{i,j} \times (d_j -c_{i,j})} & \Id_{c_{i,j}}\end{bmatrix}.
\]
Therefore,
\begin{align*}
  & \prox_{ \pas_j \, \bar{g}_{i,j}}\left( \barA \paramvect_j - \pas_j \barA^{-\top} \nabla_j f(\paramvect) \right) \\
  &  = \begin{bmatrix} \U \paramvect_j - \pas_j (\U \U^\top)^{-1} \U \nabla_j f(\paramvect) \\
    \prox_{ \pas_j \, g_{i,j}} \left(  \A \paramvect_j - \pas_j (\A \A^\top)^{-1} \A \nabla_j
      f(\paramvect)\right) \end{bmatrix}.
\end{align*}
Now, let us apply $\barA^{-1}$; by \eqref{eq:def:invbarA}, we have 
\begin{align*}
  & \barA^{-1} \, \prox_{ \pas_j \, g_{i,j}(\A \barA^{-1} \cdot)}\left( \barA \paramvect_j - \pas_j \barA^{-\top} \nabla_j f(\paramvect) \right) \\
  & =   \U^\top (\U \U^\top)^{-1} \left( \U \paramvect_j - \pas_j  (\U \U^\top)^{-1} \U \nabla_j f(\paramvect) \right)  \\
  & +  \A^\top (\A \A^\top)^{-1}  \prox_{ \pas_j \, g_{i,j}} \left(  \A \paramvect_j - \pas_j (\A \A^\top)^{-1} \A \nabla_j
    f(\paramvect)\right).
\end{align*}
Finally, since $\barA^{-1}  \barA= \Id_{d_j}$, we have 
\[
\U^\top (\U \U^\top)^{-1} \U + \A^\top (\A \A^\top)^{-1} \A =
\Id_{d_j}
\]
and this concludes the proof of \eqref{eq:pgdual-pgd}. \\ When $\A
\A^\top = \Id_{c_{i,j}}$, we have $\widetilde{\modif{\boldsymbol{\Omega}}}_{i,j} = \A^\top
\A$ and $\A \widetilde{\modif{\boldsymbol{\Omega}}}_{i,j} = \A$.  When $\U \U^\top =
\Id_{d_j-c_{i,j}}$, we have $\modif{\boldsymbol{\Omega}}_{i,j} = \U \U^\top$ and $ \Id_{d_j} -
\modif{\boldsymbol{\Pi}}_{i,j} = \U^\top \U$; this yields
\[
(\Id_{d_j} - \modif{\boldsymbol{\Pi}}_{i,j}) \modif{\boldsymbol{\Omega}}_{i,j} = \U^\top \U \U^\top \U =
\U^\top \U = \Id_{d_j} - \modif{\boldsymbol{\Pi}}_{i,j},
\]
leading to the final result of Theorem~\ref{theo:driftdual}.

\subsection{Detailed proof of \Cref{prop:MCMCsampler}}
\renewcommand{\thealgocf}{4}
Let a finite set $\Sset$ of indices. For any $\paramvect \in \Dset$,
let $\{\rho_\iota(\paramvect), \iota \in \Sset \}$ be a weight
function: $\sum_{\iota \in \Sset} \rho_\iota(\paramvect) =1$ and
$\rho_\iota(\paramvect) \geq 0$.  Finally, for any $\iota \in \Sset$,
let $q_\iota(\paramvect, \paramvect') \rmd \param'$ be a Markov
transition with respect to the Lebesgue measure on $\rset^d$. {\tt
  PGdec} and {\tt PGdual} are special instances of
\autoref{algo:generalPGD}. \begin{algorithm} \KwData{$N_{\mathrm{max}}
    \in \nset_\star$, $\paramvect^0 \in \Dset$} \KwResult{A
    $\Dset$-valued sequence $\{\paramvect^n, n \in [N_{\mathrm{max}}]
    \}$} \For{$n=0, \ldots, N_{\mathrm{max}}-1$}{ Sample $\iota \in
    \Sset$ with distribution $\{\rho_i(\paramvect^{n}), i \in \Sset
    \}$ \; Draw $\paramvect^{n+1/2} \sim q_\iota(\paramvect^{n},
    \cdot)$ \; Set $\paramvect^{n+1} = \paramvect^{n+1/2}$ with
    probability $\alpha_\iota(\paramvect^n, \paramvect^{n+1/2})$
    \begin{equation*} \label{eq:ARratio:GPGD}
\alpha_\iota(x, y) \eqdef 1 \wedge \frac{\pi(y)}{\pi(x)}
\frac{\rho_\iota(y)}{\rho_\iota(x)} \frac{q_\iota(y, x)}{q_\iota(x,y)}
    \end{equation*}
    and $\paramvect^{n+1} = \paramvect^n$ otherwise.  
  }
  \caption{General Blockwise Metropolis-Hastings. \label{algo:generalPGD}}
\end{algorithm} 
They correspond to the case $\Sset \eqdef \{ (i_1, \ldots, i_J), i_j
\in \{1, \ldots, I_j\} \}$; $\rho_\iota(\paramvect) = 1/(I_1 I_2
\ldots I_J)$ for any $\paramvect$; and to
\[
q_\iota(\paramvect, \paramvect') = \prod_{j=1}^J q_{i_j,j}(\paramvect,
\paramvect_j'), \quad \paramvect \in \Dset, \paramvect' \in \Dset
\]
where $\iota=(i_1, \ldots, i_J)$ and $\paramvect' = (\paramvect'_1,
\ldots, \paramvect'_J)$. \\
\noindent {\bf Claim1. }
Assume: {\bf [B1]} for any $\paramvect,\paramvect' \in \Dset$, there
exists $\iota \in \Sset$ such that $\rho_\iota(\paramvect)
q_\iota(\paramvect,\paramvect') \wedge \rho_\iota(\paramvect')
q_\iota(\paramvect',\paramvect)>0$; {\bf [B2]} $\pi$ is continuous on
$\Dset$; {\bf [B3]} for any compact set $K$ of $\Dset$, $\inf_{K
  \times K} \sum_{\iota \in \Sset} \rho_\iota q_\iota >0$.  Then the
sequence $\{\paramvect^n, n \geq 0 \}$ obtained by
\autoref{algo:generalPGD} is a Markov chain, taking values in
$\Dset$. It is $\phi$-irreducible, strongly aperiodic and $\pi$ is its
unique invariant distribution.
\begin{proof}
{\it   $\bullet$ $\paramvect^n \in \Dset$ for any $n$.} The proof is by induction
on $n$. This property holds true for $n=0$. Assume that $\paramvect^n
\in \Dset$. If $\paramvect^{n+1/2} \notin \Dset$, then
$\pi(\paramvect^{n+1/2}) =0$ and $\alpha_i(\paramvect^n,
\paramvect^{n+1/2}) = 0$, so that $\paramvect^{n+1} = \paramvect^n$
and $\paramvect^{n+1}$ is in $\Dset$. This concludes the
induction. \\ {\it   $\bullet$ $\pi$ is an invariant probability measure.}
Conditionally to $\iota$ and $\paramvect^n$, the distribution of
$\paramvect^{n+1}$ is
\begin{multline*}
  P_\iota(\paramvect^n, \rmd \paramvect') \eqdef \delta_{\paramvect^n}(\rmd \paramvect') \left(1 - \int_{\rset^d} \alpha_\iota (\paramvect^n, \tau) q_\iota(\paramvect^n, \tau) \rmd \tau  \right) \\
  + \alpha_\iota (\paramvect^n, \paramvect') q_\iota(\paramvect^n, \paramvect')
  \rmd \paramvect' \eqsp;
\end{multline*}
$\delta_x(\rmd \paramvect')$ denotes the Dirac mass at $x$.
Conditionally to $\paramvect^n$, the distribution of $\iota$ is
$\{\rho_i(\paramvect^n), i \in \Sset \}$. Hence the conditional
distribution of $\paramvect^{n+1}$ given $\paramvect^n$ is
\[
P_\star(\paramvect^n, \rmd \paramvect') \eqdef \sum_{i \in \Sset} \rho_i(\paramvect^n) P_i(\paramvect^n, \rmd \paramvect')
\eqsp.
\]
Following the sames lines as in \cite[Theorem
  7.2]{robert:casella:2005} (details are omitted), the {\em detailed
  balance condition} with $\pi$ can be established
\begin{multline*}
\pi(\paramvect) \sum_{i \in \Sset} \rho_i(\paramvect)
\alpha_i(\paramvect, \paramvect') q_i(\paramvect, \paramvect') \\
= \pi(\paramvect') \sum_{i
  \in \Sset} \rho_i(\paramvect') \alpha_i(\paramvect', \paramvect)
q_i(\paramvect', \paramvect) \eqsp;
\end{multline*}
hence $\pi$ is invariant for $P_\star$.  \\ {\it $\bullet$
  Irreducibility.} By [B1], the chain is $\phi$-irreducible (see
\cite[Lemma 1.1.]{mengersen:tweedie:1996}). \\ {\it $\bullet$
  Aperiodicity.} Let us prove that the compact sets are $1$-small and
the chain is aperiodic; the proof is on the same lines as the proof of
\cite[Lemma 1.2.]{mengersen:tweedie:1996}. Let $K$ be a compact set in
$\Dset$. Since $\pi < \infty$ on $\Dset$ then $\sup_K \pi < \infty$
by [B2]. For any measurable set $A \subseteq K$ and any $\paramvect
\in K$, it holds
\begin{align*}
P_\star(\paramvect,A) & \geq \sum_{\iota \in \Sset} \rho_\iota(\paramvect)
\int_A q_\iota(\paramvect, \paramvect') \alpha_\iota(\paramvect, \paramvect') \rmd
\paramvect' \\ & \geq \int_A \sum_{\iota \in \Sset}
\frac{\rho_\iota(\paramvect) q_\iota(\paramvect, \paramvect') }{\pi(\paramvect')}
\wedge \frac{\rho_\iota(\paramvect') q_\iota(\paramvect', \paramvect)
}{\pi(\paramvect)} \pi(\paramvect') \rmd \paramvect' \\ & \geq \frac{\inf_{K
    \times K} \sum_{\iota \in \Sset} \rho_\iota q_\iota }{ \sup_K \pi}
\int_A \pi(\paramvect') \rmd \paramvect'.
  \end{align*}
The RHS is positive by [B3] and this proves that $K$ is $1$-small and
the chain is aperiodic. \\ {\it $\bullet$ Unique invariant probability
  distribution.}  Finally, \cite[Propositions 10.1.1. and
  10.4.4]{meyn:tweedie:2009} prove that $\pi$ is the unique invariant
distribution.
\end{proof}
\noindent {\bf Claim 2.} Both {\tt PGdec} and {\tt PGdual} satisfy
    [B1,B2,B3]. The {\tt PGdec} and {\tt PGdual} chains are positive
    Harris-recurrent Markov chains: they satisfy a strong law of large
    numbers for any initial value in $\Dset$.
\begin{proof}
{\it $\bullet$ Both algorithms satisfy [B1].} For both algorithms,
$\rho_\iota(\paramvect) =1 /(I_1 \cdots I_J)$ and $q_\iota(\paramvect,
\paramvect')$ is proportional to 
  \[
 \prod_{j=1}^J \exp\left( - 0.5 (\paramvect'_j -
 \stepalgo_{i_j,j}(\paramvect))^\top \C_{i_j,j}^{-1} (\paramvect'_j -
 \stepalgo_{i_j,j}(\paramvect)) \right)
  \]
where $\iota = (i_1, \cdots, i_J)$ and $\stepalgo$ is
$\stepalgo^{\mathrm{PGdec}}$ or
$\stepalgo^{\mathrm{PGdual}}$. Therefore, since $\stepalgo_{i,j}(\tau) <
\infty$ for any $\tau \in \Dset$, we have $q_\iota(\paramvect,
\paramvect') \wedge q_\iota(\paramvect', \paramvect) > 0$ for any
$\iota \in \Sset$ and $\paramvect, \paramvect' \in \Dset$. \\ {\it
  $\bullet$ Both algorithms satisfy [B3].} For any compact set $K$ of
$\Dset$, we have $\sup_K \|\stepalgo_{i,j} \|< \infty$; in addition, $\stepalgo$
is a continuous function on $\Dset$ (the function $f$ is continuously
differentiable and the proximal operator is continuous by
\cite[Proposition 12.28]{Bauschke:2011ta}. Hence $\inf_{K \times K}
q_\iota>0$ and [B3] holds. \\ {\it $\bullet$ Positive Harris
  recurrence.} From Claim 1, the {\tt PGdec} Markov chain and the {\tt
  PGdual} one are positive recurrent (they are $\phi$-irreducible with
an invariant distribution, and recurrent by \cite[ Proposition
  10.4.4]{meyn:tweedie:2009}). Following the same lines as in
\cite[Theorem 8]{roberts:rosenthal:2006}, we prove that the chain is
Harris recurrent by showing that for any measurable set $A$ such that
$\int_A \pi(\paramvect) \rmd \paramvect =1$ and any $\paramvect \in
\Dset$, $\PP_\paramvect(\tau_A < \infty)=1$ where $\tau_A$ is the
return-time to the set $A$ (\cite[Theorem
  6(v)]{roberts:rosenthal:2006}); here $\PP_\paramvect$ denotes the
probability on the canonical space of the Markov chain with initial
distribution the Dirac mass at $\paramvect$ and with kernel
$P_\star$. Let $A$ be a measurable subset of $\Dset$ such that $\int_A
\pi(\paramvect) \rmd \paramvect =1$. Let $\paramvect \in \Dset$.  We
write the kernel $P_\star$ as follows
\[
P_\star(\paramvect, A) = (1-r(\paramvect)) M(\paramvect,A) + r(\paramvect)
\delta_\paramvect(A),
\]
where $r(\paramvect) \eqdef 1 - \sum_{\iota \in \Sset}
\rho_\iota(\paramvect) \int_\Dset q_\iota(\paramvect, \paramvect')
\alpha_\iota(\paramvect, \paramvect') \rmd \paramvect'$, and
\[
M(\paramvect,A) \eqdef (1-r(\paramvect))^{-1} \ \sum_{\iota \in \Sset}
\rho_\iota(\paramvect) \int_A q_\iota(\paramvect, \paramvect')
\alpha_\iota(\paramvect, \paramvect') \rmd \paramvect'.
\]
Hence, $P_\star(\paramvect, \cdot)$ is a mixture of two distributions: a
Dirac mass at $\paramvect$ and $M(\paramvect, \cdot)$. Since $\int_{A^c}
\pi(\paramvect) \rmd \paramvect =0$ (here, $A^c \eqdef \Dset \setminus A$),
then the Lebesgue measure of $A^c$ is $0$. This implies that
$M(\paramvect,A^c)=0$ and $M(\paramvect, A) =1$. It holds
\begin{align*}
\PP_\paramvect(\tau_A =+ \infty) & = \PE_\paramvect[ \1_{X_1 \notin A}
  \PP_{X_1}( \tau_A = + \infty)] \\ & = \PE_\paramvect[ \1_{X_1 \in A^c}
  \PP_{X_1}( \tau_A = + \infty)] \\ &= r(\paramvect) \, \PP_\paramvect(\tau_A
=+ \infty);
\end{align*}
indeed, starting from $\paramvect$, the chain can not reach $A^c$ when
the kernel $M(\paramvect, \cdot)$ is selected; and remains at
$\paramvect$ when this kernel is not selected. Since $r(\paramvect)
<1$ (otherwise the chain can not be $\phi$-irreducible), we have
$\PP_\paramvect(\tau_A = +\infty) =0$. This concludes the
proof. \\ {\tt $\bullet$ Strong Law of Large numbers.} A positive
Harris recurrent chain satisfies a strong law of large numbers
whatever the initial value in $\Dset$ (\cite[Theorem
  17.0.1]{meyn:tweedie:2009}).
\end{proof}

\clearpage
\newpage
 
\ifCLASSOPTIONcaptionsoff
  \newpage
\fi



%



%
%
%
%
%
%


\newoutputstream{stream}
\openoutputfile{counters}{stream}
\addtostream{stream}{
  \protect\setcounter{equation}{\arabic{equation}}%
  \protect\setcounter{section}{\arabic{section}}
}
\closeoutputstream{stream}

\newoutputstream{stream2}
\openoutputfile{countersSection}{stream2}
\addtostream{stream2}{
  \protect\setcounter{section}{\arabic{section}}%
}
\closeoutputstream{stream2}

\end{document}